%% file: main.tex
\documentclass{article}

\usepackage{custompkg}
\usepackage{customath}

\usepackage[preprint]{icml2026}
\makeatletter
\renewcommand{\ICML@preprint}{%
  \textit{Preprint. October 8, 2026.}%
}
\makeatother

\icmltitlerunning{Hierarchical Time Series Forecasting with Operator-Valued Kernels}

\author{%
  \textbf{Junjie Yang}\\
  LPSM, Sorbonne Université \& Califrais\\
  \href{mailto:junjie.yang@sorbonne-universite.fr}{\texttt{junjie.yang@sorbonne-universite.fr}}
  \And
  \textbf{Adeline Fermanian}\\
  Califrais\\
  \href{mailto:adeline.fermanian@califrais.fr}{\texttt{adeline.fermanian@califrais.fr}}
  \AND
  \textbf{Claire Boyer}\\
  LMO, Université Paris-Saclay\\
  \href{mailto:claire.boyer@universite-paris-saclay.fr}{\texttt{claire.boyer@universite-paris-saclay.fr}}
  \And
  \textbf{Gérard Biau}\\
  LPSM, Sorbonne Université\\
  \href{mailto:gerard.biau@sorbonne-universite.fr}{\texttt{gerard.biau@sorbonne-universite.fr}}
}

\makeatletter
\newcommand{\printauthorblock}{%
  \begingroup
  \def\And{%
    \end{tabular}\hfil\linebreak[0]\hfil
    \begin{tabular}[t]{c}\ignorespaces
  }%
  \def\AND{%
    \end{tabular}\hfil\\[1.25em]
    \begin{tabular}[t]{c}\ignorespaces
  }%
  \begin{center}
    \begin{tabular}[t]{c}\@author\end{tabular}
  \end{center}
  \endgroup
}
\makeatother

\icmlcorrespondingauthor{Junjie Yang}{junjie.yang@sorbonne-universite.fr}

\begin{document}

\twocolumn[
  \icmltitle{Hierarchical Time Series Forecasting with Operator-Valued Kernels}
  
  \printauthorblock

  \vskip 0.3in
]


\printAffiliationsAndNotice{}  

\input{0-Abstract.tex}

\input{1-Introduction.tex}
\input{2-Preliminaire.tex}
\input{4-Predictor.tex}
\input{5-Theory.tex}
\input{6-Experiments.tex}
\input{7-Conclusion.tex}

\bibliography{main}
\bibliographystyle{apalike}

\newpage
\appendix
\onecolumn

\input{9-Appendix.tex}


\end{document}

%% file: 0-Abstract.tex
\begin{abstract}
  Hierarchical time series forecasting requires generating accurate predictions that are coherent across multiple levels of aggregation. Existing methods either rely on two-stage pipelines with post-hoc reconciliation or on end-to-end models that often lack statistical guarantees. We introduce a unified framework that formulates hierarchical forecasting as a multivariate regression problem in a vector-valued reproducing kernel Hilbert space. By designing structured operator-valued kernels, our approach enforces coherence directly within the hypothesis space, ensuring that all predictions satisfy hierarchical constraints by construction. We further propose tailored kernel constructions for hierarchical data and derive excess risk bounds that characterize generalization performance. Empirical results on synthetic and real-world datasets demonstrate the relevance of the proposed method.
  
\end{abstract}

%% file: 1-Introduction.tex
\section{Introduction}
A \textit{hierarchical time series} (HTS) is a collection of univariate time series organized across multiple levels of aggregation, where higher-level series are obtained by summing lower-level series according to a predefined structure \citep{hyndman_forecasting_2018}. A typical example arises in retail: daily demand can be recorded at the individual product level, aggregated into product categories, and further combined to form a total demand series. The problem of \textit{hierarchical time series forecasting}, first introduced by \citet{athanasopoulos_hierarchical_2009} in the context of Australian domestic tourism, has recently attracted significant attention due to its relevance in a wide range of real-world applications. These include retail sales forecasting \citep{mancuso_machine_2021, makridakis_m5_2022a}, energy consumption forecasting \citep{nespoli_hierarchical_2020}, and road traffic forecasting \citep{ben_taieb_regularized_2019}, among others.

A key challenge in hierarchical time series forecasting is that one must not only produce accurate forecasts for each individual series, but also ensure that these forecasts are \textit{coherent}, i.e., they satisfy the aggregation constraints imposed by the hierarchy. Early approaches to this problem predominantly adopt a two-stage framework \citep[e.g.,][]{athanasopoulos_hierarchical_2009, hyndman_optimal_2011,ben_taieb_regularized_2019, wickramasuriya_optimal_2019, zhou_sloth_2023}. In the first stage, base forecasts are generated independently for each series using standard univariate time series models, such as ETS \citep{brown_statistical_1959} or ARIMA \citep{box_time_2015}. In the second stage, these forecasts are adjusted to enforce coherence through a procedure known as reconciliation. More recently, several works \citep[e.g.,][]{rangapuram_endtoend_2021, tsiourvas_learning_2024, cini_graph-based_2024, doumeche_forecasting_2025} have proposed to integrate base forecasting and reconciliation within a single end-to-end framework. In these approaches, model parameters are jointly learned across all series, enabling the direct generation of coherent forecasts without requiring a separate post-processing step. However, most of these methods rely on deep neural network architectures, e.g., DeepVAR \citep{salinas_high-dimensional_2019} or Autoformer \citep{wu_autoformer_2021}, which can be difficult to train effectively in time series settings where data availability is limited. Furthermore, generalization guarantees remain limited for these end-to-end methods.

In this work, we recast the HTS forecasting problem as a (potentially non-linear) multivariate regression problem and introduce a unified kernel framework in which predictors are learned in a vector-valued reproducing kernel Hilbert space \citep[vv-RKHS,][]{pedrick_theory_1957, schwartz_sousespaces_1964}. This approach is motivated by several key advantages. First, by constructing operator-valued kernels \citep[OVKs, ][]{carmeli_vector_2006} whose associated vv-RKHSs consist only of functions taking values in the coherent subspace, we enforce coherence by construction. Second, the regularized least-squares problem has a unique global minimizer \citep[see, e.g.,][]{micchelli_learning_2005}. Finally, the well-established statistical learning theory for RKHSs \citep[see, e.g.,][]{smale_learning_2007, caponnetto_optimal_2007} allows us to derive generalization guarantees for our proposed predictor.

Our paper is organized as follows. In \Cref{sec:preliminaries}, we introduce the necessary background on hierarchical time series forecasting and operator-valued kernels. 
In \Cref{sec:predictor}, we introduce our predictor RVVKRR, based on vector-valued kernel ridge regression \citep{micchelli_learning_2005}, and provide efficient computational methods for both separable and finite-feature non-separable OVKs. In particular, we characterize the OVKs whose associated vv-RKHS, used as the hypothesis space, contains only functions that produce coherent forecasts.
In \Cref{sec:stat_analysis}, using integral operator theory in vv-RKHSs, we derive an excess risk bound for the proposed predictor under an i.i.d. regression model, explicitly highlighting the roles of both the hierarchical structure and the OVKs.
In \Cref{sec:exper}, we analyze the behavior of the proposed predictor on a synthetic dataset, and demonstrate its performance on two real-world datasets: TourismLarge and Traffic.

%% file: 2-Preliminaire.tex
\section{Preliminaries}
\label{sec:preliminaries}
\subsection{Hierarchical Time Series Forecasting}
\label{bg:hts}
A hierarchical time series is a multivariate time series $(y_t)_{t=1}^T \in (\reals^d)^T$ that satisfies linear aggregation constraints, which can be typically encoded into a tree hierarchy \citep{hyndman_optimal_2011}. Here, $d \in \intsplus$ represents the number of univariate time series. 
It can be divided into a \textit{bottom-level series} $(y_t^b)_{t=1}^T$ with $y_t^{b} \in \reals^{d_b}$ (i.e., the time series at the leaf nodes of the hierarchy), and an \textit{aggregated series} $(y_t^a)_{t=1}^T$ with $y_t^{a} \in \reals^{d_a}$  (i.e., those of the remaining nodes). Here, we have clearly $d = d_a + d_b$. Formally, at each timestamp $t$, the vector representing the whole time series $y_t$ can be split into two such that $y_t = \begin{bmatrix} (y_t^a)^\trans & (y_t^b)^{\trans} \end{bmatrix}^{\trans}$ and the whole time series can be obtained by the following transformation of the bottom-level series $y_t^b$:
\begin{equation}
  y_t = S y_t^b  \quad \text{with}\; S = \begin{bmatrix} \mat S_{\mathrm{sum}} \\ \idmat{d_b} \end{bmatrix} \in \set{0, 1}^{d \times d_b},
\label{eq:aggre_matrix}
\end{equation}
where $\mat S$ is called the aggregation matrix and $ \mat S_{\mathrm{sum}} \in \set{0, 1}^{d_a \times d_b}$ is called the summation matrix. It should be noted that the summation matrix describes exactly the hierarchical structure of the time series. A time series that satisfies those constraints is said to be coherent, i.e., $y_t \in \im(\mat  S)$ for all $t \in \intset{T}$.
An example of such a hierarchical time series is provided in \Cref{app:hts_example}.

\paragraph{Hierarchical Forecasting Methods.} In the literature on hierarchical forecasting, to generate the coherent forecast $\hat{y}_t \in \im(\mat S)$, a two-stage approach is typically followed. In the first stage, the multivariate and hierarchical structure of the data is completely ignored. A separate predictor is trained for each of the $d$ univariate time series, with all predictors learned independently. Any classical time series forecasting method can be used at this stage, such as exponential smoothing \citep[ETS,][]{brown_statistical_1959}, 
or kernel ridge regression \citep[KRR,][]{saunders_ridge_1998}, to name a few. The resulting forecast of the multivariate time series at time step $t$, referred to as the \textit{base forecast}, is denoted by $\tilde{y}_t \in \reals^d$, which is not necessarily coherent. The second stage, called \textit{reconciliation}, transforms the base forecast $\tilde{y}_{t} \in \reals^d$ into a coherent forecast $\hat{y}_{t} \in \im(\mat  S) \subset \reals^d$. Most reconciliation methods in the literature adopt a linear transformation $\hat{y}_{t} = \mat S \mat G \tilde{y}_{t}$ where $\mat G \in \reals^{d_b \times d}$ is a \textit{bottom-extractor} matrix or a \textit{reconciliation} matrix that forms only the bottom-level time series from the full set of forecasts across all hierarchy levels. Different choices of $\mat G$ correspond to different reconciliation strategies, detailed in \Cref{app:reconciliation}.

\textit{End-to-end} approaches \citep[e.g.,][]{rangapuram_endtoend_2021,tsiourvas_learning_2024, doumeche_forecasting_2025} integrate reconciliation  directly into the learning procedure. Typically, one chooses a hypothesis space $\mathcal{H} \subset \funcset{\mathcal X}{(\reals^d)}$ where $\mathcal X$ denotes the feature space, and parameterizes the predictor as $f(\cdot) = \mat S \mat G h(\cdot)$ with $h \in \mathcal{H}$. Learning is performed by 
solving an empirical risk minimization problem.
For example, \citet{rangapuram_endtoend_2021, tsiourvas_learning_2024} parameterize $\mathcal{H}$ by neural network architectures, such as DeepVAR \citep{salinas_high-dimensional_2019} or Autoformer \citep{wu_autoformer_2021}.

\subsection{Operator-Valued Kernels}
\label{sec:ovk}
Let $\mathcal{X}$ be a nonempty set, $\mathcal{Y}$ be a real separable Hilbert space with inner space $\inner{\cdot}{\cdot}{\mathcal{Y}}$, and $\lineop{\mathcal{Y}}$ be the space of bounded linear operators from $\mathcal{Y}$ to $\mathcal{Y}$. Operator-valued kernels (OVKs), defined as follows \citep[see, e.g.,][]{carmeli_vector_2006}, can be seen as the generalization of scalar-valued kernels.
\begin{definition}[Operator-Valued Kernels]
  A function $\ovk: \mathcal X \times \mathcal X \to \lineop{\mathcal{Y}}$ is said to be an operator-valued kernel if it is self-adjoint: for all $(x_i, x_j) \in \mathcal X^2$, $\ovk(x_i, x_j) = \ovk(x_j, x_i)^{\adj}$ where the superscript $^{\adj}$ denotes the adjoint of a bounded linear operator;
  and positive: for all $N \in \intsplus$ and $y \in \mathcal{Y}$,
  \begin{equation*}
      \sum_{i=1}^N\sum_{j=1}^N a_i a_j\inner{\ovk(x_i, x_j)y}{y}{\mathcal{Y}} \geq 0
  \end{equation*}
    for all $(x_1, \dots, x_N) \in \mathcal X^N$ and $(a_1, \dots, a_N) \in \mathcal \reals^N$.
\end{definition}
The relevance of operator-valued kernels arises from their connection to function spaces known as vector-valued reproducing kernel Hilbert spaces \citep[vv-RKHS, ][]{pedrick_theory_1957, schwartz_sousespaces_1964}, defined as follows.
\begin{definition}[Vector-valued RKHSs]
  A Hilbert space of functions $\mathcal H \subset \funcset{\mathcal X}{\mathcal Y}$ is called a $\mathcal Y$-valued reproducing kernel Hilbert space on $\mathcal X$ if for all $x \in \mathcal X$ there exists a positive constant $M_x$ such that $\norm{f(x)}{\mathcal Y} \leq M_x \norm{f}{\mathcal H}$ for all $f \in \mathcal{H}$.
\end{definition}
Given an operator-valued kernel $\mathcal{K}$, one can show \citep[see, e.g.,][Proposition 2.1]{carmeli_vector_2006} that there is a unique $\mathcal{Y}$-valued reproducing kernel Hilbert space (RKHS) $\vvrkhs \subset \funcset{\mathcal{X}}{\mathcal{Y}}$ such that for all $x \in \mathcal{X}$ and for all $f \in \vvrkhs$, 
\begin{equation}
  f(x) = \ovk_{x}^{\adj}f \quad \text{(\textit{reproducing property})},
\end{equation}
where the linear operator $\ovk_{x}: \mathcal{Y} \to \vvrkhs$ is defined by 
$\ovk_{x}y (x') = \ovk(x', x) y$,
for all $x' \in \mathcal{X}$ and for all $y \in \mathcal{Y}$.
The converse is true by using the Riesz representation theorem \citep{carmeli_vector_2006}.

A special class of OVKs, called the separable OVKs \citep[see e.g.,][Example 1]{caponnetto_universal_2008}, are often employed in real-world applications:
\begin{equation}
  \ovk(x, x') = \kernel(x, x') \op A,
\end{equation}
where $\kernel: \mathcal X \times \mathcal X \to \reals$ is a scalar-valued kernel defined on the input space $\mathcal X$ and the operator $\op A \in \lineop{\mathcal{Y}}$, referred to in some literature as the task kernel, is self-adjoint and positive, i.e., $\op A = \op A ^{\adj}$ and $\inner{\op A y}{y}{\mathcal{Y}} \geq 0$ for all $y \in \mathcal{Y}$. More details on the properties of separable OVKs are provided in \Cref{app:ovk_sep}.

OVKs and vv-RKHSs have been employed to address a variety of supervised learning problems in which the output space can be modeled as a real separable Hilbert space. See \Cref{app:ovk_sup} for a detailed discussion. In this paper, the predictor we introduce for hierarchical time series forecasting is a function of a $\im(S)$-valued RKHS.

%% file: 4-Predictor.tex
\section{Reconciled Forecasting with VVKRR}
\label{sec:predictor}
In this section, we first propose to construct a reconciled hierarchical time series predictor, $\hat{f}: \mathcal{X} \to \im(\mat S)$, using vector-valued kernel ridge regression (VVKRR). This approach, originally introduced by \citet{micchelli_learning_2005}, leverages OVKs to address regression problems with multidimensional outputs. We then present several strategies for designing appropriate OVKs tailored to the hierarchical time series setting.

\subsection{Reconciled Vector-Valued Kernel Ridge Regression}
\label{sec:vvkrr}
Assume that we have observed a hierarchical time series along with its features, denoted as $\left( (x_1, y_1), \dots, (x_T, y_T)\right) \in \left(\mathcal X \times \im(\mat S)\right)^T$. For instance, the feature vector $x_t \in \reals^{d_b p+q}$ with $(p, q) \in (\intsplus)^2$ could include $p$ lagged observations of each of the $d_b$ bottom-level series, along with $q$ additional covariates capturing temporal effects such as seasonality or calendar-based features. $\im(\mat S)$ is a finite-dimensional subspace of $\reals^d$, thus a real Hilbert space with the inner product inherited from $\reals^d$. Moreover, one can show that $\dim( \im(\mat S)) = d_b$ where $d_b$ is the number of bottom-level series (see \Cref{lemma:dim_im_S}).

As a result, it is natural to design an OVK of the type $\ovk: \mathcal X \times \mathcal X \to \lineop{\im(\mat S)}$ and consider the vv-RKHS $\vvrkhs \subset \funcset{\mathcal X}{\im(\mat S)}$ associated to it as the hypothesis space of our predictor. We choose to find such a function $\hat{f}$ by solving the following optimization problem:
\begin{equation} \label{eq:vvkrr_optim}
  \hat{f} = \argmin_{f\in \vvrkhs} \frac{1}{T} \sum_{t=1}^T \norm{f(x_t) - y_t}{\reals^d}^2 + \lambda \norm{f}{\vvrkhs}^2,
\end{equation}
where $\lambda > 0$ is a regularization parameter. In the rest of the paper, we refer to such a predictor as the \textit{reconciled vector-valued kernel ridge regression} (RVVKRR).

By the representer theorem  \citep[Theorem 4.1]{micchelli_learning_2005}, we know that $\hat{f}$ takes the form $\hat{f}= \sum_{t=1}^T \ovk_{x_t} c_t $
with $(c_t)_{t=1}^T \in (\im(\mat S))^T \subset (\reals^{d})^T$ satisfying the following linear system:
\begin{equation}
  T \lambda c_t + \sum_{j=1}^T \ovk(x_t, x_j) c_j = y_t.
\label{eq:ls_vvkrr}
\end{equation}
For separable OVKs, \Cref{eq:ls_vvkrr} can be reduced to a discrete-time Sylvester equation that can be solved by the Hessenberg-Schur method \citep{golub_hessenberg-schur_1979, sima_algorithms_1996}. When a non-separable OVK admitting an explicit finite-dimensional feature map is employed, we solve the equivalent primal problem using matrix-free preconditioned conjugate gradient method \citep{hestenes_methods_1952,cutajar_preconditioning_2016}. More details on the computational methods are provided in \Cref{app:compu_vvkrr}.
It can be shown that $\hat{f}$ can also be written as $\hat{f} = (\hat{\Sigma}_{\ovk} + \lambda \idop)^{-1}\frac{1}{T}\sum_{t=1}^{T}\ovk_{x_t}y_t$
with $\hat{\Sigma}_{\ovk} \coloneqq \frac{1}{T}\sum_{t=1}^T\ovk_{x_t}\ovk_{x_t}^{\adj}$ \citep[see e.g.,][Proposition 1]{caponnetto_optimal_2007}. While the first expression of the predictor is more suitable from a computational point of view, the latter is more commonly used in the statistical analysis of the predictor, as will be discussed further in \Cref{sec:stat_analysis}.

\subsection{Coherent Operator-Valued Kernels for Hierarchical Time Series.}
\label{sec:ovk_hts}
In this section, we introduce two complementary constructions of operator-valued kernels tailored to hierarchical time series. First, we characterize separable coherent OVKs and show how to derive their task kernels from similarity matrices that need not be positive semi-definite. Second, we start from an arbitrary, possibly non-separable, $\lineop{\reals^{d}}$-valued kernel and map it into the coherent output space through a reconciliation operator. 

\subsubsection{Intrinsic Construction of Separable Coherent OVKs}
As we have shown in \Cref{sec:ovk}, the task kernel $\mat A$ of a separable OVK $\ovk(x, x') = \kernel(x, x') \mat A$ must be a self-adjoint and positive operator in $\lineop{\im(\mat S)}$. In this section, we propose a method to construct such a task kernel $\mat A$ from a symmetric similarity matrix $\mat M \in \mathbb{S}^{d}$ that encodes the pairwise similarities between the time series. First, we will characterize all coherent task kernels $\mat A$ by the following lemma, whose proof can be found in \Cref{sec:proof_lemma_ovk_image_s}.

\begin{lemma} \label{lemma:ovk_image_s}
  Let $(u_1, \dots, u_{d_b})$ be an orthonormal basis for $\im(\mat S)$ with $u_{i} \in \reals^d$ for all $i \in \intset{d_b}$ (for example, obtained by the singular value decomposition of the matrix $\mat S$). Let $\mat{U}_{\mat S} = \begin{bmatrix}
    u_1 & u_2 & \cdots  & u_{d_b}
\end{bmatrix}\in \reals^{d\times d_b}$. Then any self-adjoint and positive operator $\mat A \in \lineop{\im(\mat S)} $ admits a matrix representation
$
  \mat{U}_{\mat S} \mat A' \mat{U}_{\mat S}^{\trans}
$
where $\mat A' \in \reals^{d_b \times d_b}$ is a symmetric and positive semi-definite matrix.
\end{lemma}
As a result, designing a coherent separable OVK is equivalent to designing a positive semi-definite task kernel $\mat A' \in \mathbb{S}^{d_b}_+$ in the $d_b$-dimensional coherent coordinate system. One can simply choose $\mat A' = \idmat{d_b}$ to obtain a coherent task kernel $\mat A = \mat{U}_{\mat S} \mat{U}_{\mat S}^{\trans}$, which is the identity operator on $\im(\mat S)$. Another simple choice for $\mat A'$ is a diagonal matrix $\mat A' = \diag\left(\lambda_1, \dots, \lambda_{d_b}\right)$ where where $\lambda_i > 0$ for all $i \in \intset{d_b}$.
However, those choices do not explicitly encode any prior knowledge about the relationships between time series, beyond the aggregation structure. In fact, it can be shown that this choice corresponds to a learning algorithm based on $d_b$ independent kernel ridge regression predictors \citep[KRR,][]{saunders_ridge_1998, evgeniou_regularization_2000}. See \Cref{sec:ovk_sep_diag} for more details.

Instead, we propose to construct the task kernel $\mat A \in \lineop{\im(\mat S)}$ from a similarity matrix $\mat M \in \mathbb{S}^{d}$, assumed to be available a priori, that encodes the pairwise similarities between the time series. Since such a matrix may not lie in $\lineop{\im(\mat S)}$ or may not be positive semi-definite, we seek a valid task kernel $\mat A$ by finding $\mat A' \in \mathbb{S}^{d_b}_+ $ that minimizes $\norm{\mat M - \mat{U}_{\mat S} \mat A' \mat{U}_{\mat S}^{\trans}}{\textrm{F}}^2$. The solution is given by the following lemma, whose proof can be found in \Cref{sec:proof_lemma_nearest_task_kernel}.
\begin{lemma} \label{lemma:nearest_task_kernel}
  The solution of the optimization problem 
  $
    \min_{\mat A' \in \mathbb{S}^{d_b}_+} \norm{\mat M - \mat{U}_{\mat S} \mat A' \mat{U}_{\mat S}^{\trans}}{\textrm{F}}^2
  $
  is given by $\mat A'_{\mathrm{opt}} = \mat Q \diag(\max(\sigma_i, 0)) Q^{\trans}$, where $Q\in \reals^{d_b \times d_b}$ and $(\sigma_i)_{i=1}^{d_b}$ are from the eigendecomposition of $\mat{U}_{\mat S}^{\trans} \mat M \mat{U}_{\mat S}$, i.e., $\mat{U}_{\mat S}^{\trans} \mat M \mat{U}_{\mat S} = \mat Q \diag(\sigma_i) Q^{\trans}$.
\end{lemma}
In the case where $\mat M \in \mathbb{S}^{d}_+$, the task kernel obtained via the above construction admits a simpler form. This result is summarized in the following corollary, whose proof is provided in \Cref{app:proof_corol_nearest_task_kernel}.
\begin{corollary} \label{corol:nearest_task_kernel}
  When $\mat M \in \mathbb{S}^{d}_+$, the minimizer of the optimization problem $\min_{\mat A} \norm{\mat M - \mat A}{\textrm{F}}^2$ with $ A \in \set{\mat{U}_{\mat S} \mat A' \mat{U}_{\mat S}^{\trans} \mid \mat A' \in \mathbb{S}^{d_b}_+}$
  is written as $\mat A_{\mathrm{opt}} = \mat P \mat M \mat P$ where $\mat P \in \reals^{d\times d}$ is the orthogonal projection matrix onto $\im(\mat S)$.
\end{corollary}
Following \Cref{corol:nearest_task_kernel}, we give an example of a coherent separable OVK constructed from the graph relationship between the time series, which will be used in the synthetic experiment of \Cref{sec:toy_experiment}.
\begin{example}[Graph-based separable OVK]
\label{example:toy_laplacian_kernel}
Let $\mat L$ be the Laplacian matrix of a graph encoding known relationships between the time series. Since $\mat L$ is positive semi-definite, we may take $\mat M=\mat L^{\dagger}$ in \Cref{corol:nearest_task_kernel}, where $\dagger$ denotes the Moore–Penrose inverse. This gives the coherent task kernel $\mat A_{\mathrm{Lap}} = \mat P\mat L^{\dagger}\mat P $
and the corresponding separable OVK
$
  \ovk_{\mathrm{Lap}}(x,x')
  = \kernel(x,x')\mat P\mat L^{\dagger}\mat P.
$
For comparison, taking $\mat M=\mat I_d$ gives $\ovk_{\mathrm{Id}}(x,x')=\kernel(x,x')\mat P$.
\end{example}

\begin{remark}
  Constructing a task kernel $\mat A$ based on the relationships between tasks modeled by a given structure is common in the multi-task learning literature. In the context where the structure is a general graph, \citet{micchelli_kernels_2004} propose considering the Moore–Penrose inverse of the graph Laplacian $L$. A non-singular variant, $\mat A = (\mu \mat L + \lambda \mat I)^{-1}$, also appeared in \citet{sheldon_graphical_2008}. When the tasks are organized in a hierarchical fashion, which is closely related to our setting, \citet{widmer_leveraging_2010} propose defining the similarity matrix between the organisms (leaf nodes) based on their distance in a taxonomic tree.
\end{remark}

\subsubsection{General Coherent OVKs by Linear Composition} \label{sec:ovk_compose}
We now consider the case where a $\lineop{\reals^{d}}$-valued kernel is available and propose a method to build a $\lineop{\im(\mat S)}$-valued kernel on top of it. In this case, we can leverage the reconciliation matrix $\mat G$. Specifically, given a $\lineop{\reals^{d}}$-valued kernel $\widetilde{\ovk}$, define
\begin{equation} \label{eq:ovk_compose}
  \ovk_{\mat S \mat G} (x, x') \coloneqq \mat S \mat G \widetilde{\ovk}(x, x') \mat G^{\trans}\mat S^{\trans}
\end{equation}
for all $(x, x') \in \mathcal{X}^2$, where $S \in \set{0, 1}^{d \times d_b}$ is the aggregation matrix defined in \Cref{eq:aggre_matrix} and $\mat G \in \reals^{d_b \times d}$ is a fixed reconciliation matrix (see \Cref{bg:hts}). One can verify that $\ovk_{\mat S \mat G}$ is a valid $\lineop{\im(\mat S)}$-valued kernel. The base kernel $\widetilde{\ovk}$ determines how inputs and tasks interact before imposing coherence, whereas $\mat S \mat G$ determines the geometry through which these relationships are transferred into the coherent output space. It should be noted that the base kernel $\widetilde{\ovk}$ will not be required to be separable. We give first an example of a non-separable base kernel for multivariate time series.

\begin{example}[Non-separable Base OVK] \label{example:nonsep_base_kernel}
Suppose that the input space can be decomposed as $\mathcal{X} = (\mathcal{X}_0)^d$. Let $x=(x_1,\ldots,x_d) \in \mathcal{X}$, where $x_r\in\mathcal X_0$ contains the covariates of series $r \in \intset{d}$. 
Let $\mat B\in\reals^{d\times q}$ be an arbitrary task embedding with nonzero rows normalized to unit Euclidean norm. For example, $\mat B_{r,:}$ could be the one-hot encoding that represents the hierarchical level of the series $r$ or other non-hierarchical labels on the series (see, e.g., \Cref{sec:tourism_experiment}). Then we define
\begin{equation}
  \mat A_\rho=(1-\rho)\mat I_d+\rho\mat B\mat B^{\trans},
  \qquad \text{with} \; \rho\in[0,1].
  \label{eq:nonsep_task_similarity}
\end{equation}
Let $\lambda_1,\ldots,\lambda_d>0$. Given two inputs $x=(x_1,\ldots,x_d)$ and $x'=(x'_1,\ldots,x'_d)$, we define a non-separable base kernel $\widetilde{\ovk}_\rho:\mathcal X\times\mathcal X\to\lineop{\reals^d}$ as follows: for any $ r,s\in\intset{d}$,
\begin{equation}
  \left[\widetilde{\ovk}_\rho(x,x')\right]_{rs}
  =
  \frac{(\mat A_\rho)_{rs}}
       {\sqrt{\lambda_r\lambda_s}}
  \kernel(x_r,x'_s),
  \label{eq:nonsep_base_kernel}
\end{equation}
where $\kernel:\mathcal X_0\times\mathcal X_0\to\reals$ is a scalar-valued kernel admitting an explicit finite-dimensional feature map $\phi:\mathcal X_0\to\reals^{p_\phi}$, i.e.,
$
  \kernel(x_r,x'_s)=\langle\phi(x_r),\phi(x'_s)\rangle_{\reals^{p_\phi}}.
$
Such a kernel allows for task-specific covariates to influence the cross-task relationships. Moreover, $\widetilde{\ovk}_\rho$ is a valid OVK, as it admits a feature operator $\Gamma_\rho$ that can be explicitly characterized. More details can be found in \Cref{app:proof_example_nonsep_base_kernel_feature_map}.
\end{example}

Now we give one example of such general coherent OVKs with the MinT reconciliation matrices.
\begin{example}[MinT]
\label{example:gen_ovk_mint}
Let $\mat W\in\mathbb S^d_{++}$ be an estimate of the covariance matrix of the base-forecast errors used in the MinT reconciliation \citep{wickramasuriya_optimal_2019}. The common MinT variants, e.g., OLS, WLS-variance and shrinkage, correspond to
\begin{gather}
  \mat W_{\mathrm{OLS}}=\mat I_d,
  \quad
  \mat W_{\mathrm{WLS}}=\diag(\widehat{\mat W}),
  \\
  \mat W_{\mathrm{shr}}
  =(1-\alpha)\widehat{\mat W}
   +\alpha\diag(\widehat{\mat W}),
  \label{eq:mint_covariance_variants}
\end{gather}
where $\widehat{\mat W}$ is the empirical covariance matrix of the base-forecast residuals and $\alpha\in(0,1]$ is the shrinkage intensity (see \Cref{app:reconciliation} for more details). Then the MinT reconciliation matrix \citep{wickramasuriya_optimal_2019} and the corresponding coherent projection are
$\mat G_{\mat W} =(\mat S^{\trans}\mat W^{-1}\mat S)^{-1} \mat S^{\trans}\mat W^{-1}$ and $ \mat P_{\mat W}=\mat S\mat G_{\mat W}$.

For any base OVK $\widetilde{\ovk}$, we define the associated MinT-composed OVK as
\begin{equation}
  \ovk_{\mat W}(x,x')
  =\mat P_{\mat W}\widetilde{\ovk}(x,x')
   \mat P_{\mat W}^{\trans}.
  \label{eq:mint_kernel}
\end{equation}
For $\mat W=\mat I_d$, $\mat P_{\mat W}$ is the orthogonal projector $\mat P=\mat S(\mat S^{\trans}\mat S)^{-1}\mat S^{\trans}$. Thus, if $\widetilde{\ovk}(x,x')=\kernel(x,x')\widetilde{\mat A}$ is separable, \Cref{eq:mint_kernel} reduces to $\kernel(x,x')\mat P\widetilde{\mat A}\mat P$ and coincides with \Cref{corol:nearest_task_kernel}.
\end{example}

Interestingly, the predictor obtained by choosing an OVK of this form coincides with that produced by an alternative learning algorithm previously studied in the kernel literature for other types of outputs (see \Cref{app:ovk_sup} for further details). Specifically, for any $\lambda > 0$, if $\hat{h}$ is the minimizer of the optimization problem:
\begin{equation}
\min_{h\in \mathcal H_{\widetilde{\ovk}}} \frac{1}{T} \sum_{t=1}^T \norm{SGh(x_t) - y_t}{\reals^d}^2 + \lambda \norm{h}{\mathcal H_{\widetilde{\ovk}}}^2,
\label{eq:rvvkrr_fixed}
\end{equation}
where $\mathcal H_{\widetilde{\ovk}}$ is the RKHS associated to $\widetilde{\ovk}$, then $\hat{f} = \mat S \mat G \circ \hat{h}$ is the minimizer of the following one:
\begin{equation}
\min_{f\in \mathcal{H}_{\widetilde{\ovk}, \mat S\mat G}} \frac{1}{T} \sum_{t=1}^T \norm{f(x_t) - y_t}{\reals^d}^2 + \lambda \norm{f}{\mathcal{H}_{\widetilde{\ovk}, \mat S\mat G}}^2,
\end{equation}
where $\mathcal{H}_{\widetilde{\ovk}, \mat S\mat G}$ is the RKHS associated to $\ovk_{\mat S \mat G}$ as defined in \Cref{eq:ovk_compose}. The proof of the observation is provided in \Cref{app:proof_prop_same_estimator_sg}. It also allows us to reinterpret the WeaKL estimators introduced by \citet{doumeche_forecasting_2025} into our framework, by viewing them as instances where the OVKs are special cases of the construction in \Cref{eq:ovk_compose}. Additional discussion on this connection is provided in \Cref{app:weakl}.

%% file: 5-Theory.tex
\section{Statistical Analysis}
\label{sec:stat_analysis}
\subsection{Statistical Setting}
\label{sec:stat_setting}
In this section, from now on, we consider hierarchical time series forecasting as a possibly nonlinear regression problem. Let $(\rv X, \rv Y_b, \rv Y)$ be three random variables, valued in $\mathcal{X}$, $\mathcal{Y}_b$ and $\mathcal{Y}$ respectively.
In the hierarchical forecasting setting, $\rv Y_b$ denotes the $d_b$ bottom-level time-series observations, $\rv Y$ denotes the $d$ time-series observations across all levels, and $\rv X$ denotes the associated covariates.
We assume that $\mathcal{X}$ is a Polish space, $\mathcal{Y}_b$ and $\mathcal{Y}$ are two real separable Hilbert spaces. We consider that the following relationships exist between the variables:
\begin{gather}
  \rv Y_b = g^{\targ}(\rv X) + \xi, \quad \rv Y = \op S \rv Y_b,
\end{gather}
where $g^{\targ}: \mathcal{X} \to \mathcal{Y}_b$ is an unknown target function and $\op S \in \lineop{\mathcal{Y}_b, \mathcal{Y}}$ is a known bounded linear operator. Here, $\xi$ is a $\mathcal{Y}_b$-valued random variable modeling the noise. We assume that $\im(S)$ is closed, so that it is a valid Hilbert space with the inner product inherited from $\mathcal{Y}$.

Let $\loi$ be the joint probability distribution of $(\rv X, \rv Y)$, and $\rho$ be its marginal distribution on $\rv X$. Given $T$ observed data points $\left( (x_1, y_1), \dots, (x_T, y_T)\right) \in (\mathcal{X} \times \mathcal{Y})^T$ independently distributed by $\loi$, the goal is to obtain an estimator $\hat{f}: \mathcal{X} \to \im(\op S)$ of $\op S \circ g^{\targ}$ whose excess risk can be controlled, defined as follows:
\begin{align}
   &\risk(\hat{f}) - \risk(\op S \circ g^{\targ}) \nonumber \\
   & \quad  = \mathbb{E} \Big[\big\|\hat{f}(\rv X) - \rv Y \big\|_{ \mathcal{Y}}^2\Big] - \mathbb{E} \Big[\big\|\op S \circ g^{\targ}(\rv X) - \rv Y \big\|_{ \mathcal{Y}}^2 \Big].
\end{align}
It should be noted that the statistical analysis in this section assumes i.i.d. input–output pairs. Extending the excess-risk bound to temporally dependent observations is left for future work.

Let $L^2(\mathcal{X}, \rho, \mathcal{Y})$ be the Hilbert space of measurable  function $f: \mathcal{X} \to \mathcal{Y}$ that are square-integrable with respect to the probability measure $\rho$, i.e., $\int_{\mathcal{X}} \norm{f(x)}{\mathcal{Y}}^2 d\rho(x) < \infty$. Such a space is endowed with the inner product defined by $\inner{f_1}{f_2}{\rho} = \int_{\mathcal{X}} \inner{f_1(x)}{f_2(x)}{\mathcal{Y}} d\rho(x)$ and the associated norm $\norm{\cdot}{\rho}$ induced by the inner product.

We make the following assumptions.
\begin{assumption} \label{assum:val}
  $\norm{g^{\targ}(\rv X)}{\mathcal{Y}_b}^2 \leq C_1 < \infty$ almost surely and $\norm{\xi}{\mathcal{Y}_b}^2 \leq C_2 < \infty$ almost surely.
\end{assumption}
\begin{assumption} \label{assum:ovk}
  $\norm{\ovk_{x}}{2}^2 = \trace{\ovk_{x}^{\adj}\ovk_{x}} \leq \kappa < \infty $ for all $x \in \mathcal{X}$.
\end{assumption}
\begin{assumption} \label{assum:noise}
  The random variables $\rv X$ and $\xi$ are independent and $\expes{ \xi} = 0$.
\end{assumption}
\begin{assumption}\label{assum:source_condition}
There exists an $h^{\targ} \in L^2(\mathcal{X}, \rho, \mathcal{Y})$ such that, for some $0 < r \leq 1$, $\op S \circ g^{\targ} = L_{\ovk}^{r} h^{\targ}$ where the integral operator $L_{\ovk}: L^2(\mathcal{X}, \rho, \mathcal{Y}) \to L^2(\mathcal{X}, \rho, \mathcal{Y})$ is defined by
  $
    L_{\ovk} h \coloneqq \int_{\mathcal{X}} \ovk_{x} h(x) d\rho(x)
  $
  for any $h \in L^2(\mathcal{X}, \rho, \mathcal{Y})$.
\end{assumption}
The source condition in the last assumption is standard in the kernel literature and controls the smoothness of the target function $\op S \circ g^{\targ}$ relative to the kernel integral operator $L_{\ovk}$.

\subsection{Main Results}
We now present a compact excess-risk bound for the RVVKRR estimator, whose proof is deferred to \Cref{sec:excess_risk_proofs}. The proof is obtained using the technique of integral operators in RKHS, particularly introduced by \citet{smale_learning_2007, caponnetto_optimal_2007}. 
To state the result, define the covariance operator and its effective dimension by 
$\Sigma_{\ovk} \coloneqq \int_{\mathcal{X}}\ovk_x\ovk_x^{\adj}\,d\rho(x)$ and
$\mathcal{N}(\lambda)
  \coloneqq
  \trace{(\Sigma_{\ovk}+\lambda\idop)^{-1}\Sigma_{\ovk}}$.

\begin{proposition}[Excess-risk Bound]
\label{prop:bound_excess_risk}
  Under \Cref{assum:val,assum:ovk,assum:noise,assum:source_condition},
  let $\delta\in(0,1)$ and suppose that $\lambda \geq 8 \log\left(\frac{4}{\delta}\right) \frac{\kappa}{\sqrt{T}}$.
  Then, with probability at least $1-\delta$,
  \begin{align*}
    &\risk(\hat{f})-\risk(\op S\circ g^{\targ}) \\
    \leq \;& \frac{128}{T}\log^2\!\left(\frac{4}{\delta}\right)
    \left(
      \mathcal{N}(\lambda)\sigma_S^2 +\frac{\kappa}{\lambda}(3B_S^2+M_S^2)
    \right) \\
& +2\lambda^{2r}R_r^2,
  \end{align*}
  with $\sigma_S^2 \coloneqq \expes{\norm{\op S\xi}{\mathcal{Y}}^2}$, $B_S^2 \coloneqq \norm{\op S}{\mathrm{op}}^2 C_1$, $ M_S^2 \coloneqq \norm{\op S}{\mathrm{op}}^2 C_2$ and $R_r \coloneqq
  \norm{L_{\ovk}^{-r}(\op S\circ g^{\targ})}{\rho}$.

\end{proposition}

The first term in \Cref{prop:bound_excess_risk} is the estimation error, controlled by the effective dimension and the transformed noise level.
The second term is the approximation error induced by regularization and the source condition.

\begin{remark} In the  HTS setting, we have
$\mathcal{Y}=\im(\mat S)\subset\reals^d$. Denote $\Sigma_\xi=\expes{\xi\xi^\top}\in\reals^{d_b\times d_b}$. Then by the definition of $\sigma_S^2$, we have
\begin{align*}
  \sigma_S^2
  \leq \norm{\mat S}{\mathrm{op}}^2
    \expes{\norm{\xi}{2}^2}
  =\norm{\mat S}{\mathrm{op}}^2\trace{\Sigma_\xi}
  \leq\norm{\mat S}{\mathrm{F}}^2\trace{\Sigma_\xi}.
\end{align*}
Thus, the estimation error in \Cref{prop:bound_excess_risk} depends on both the total bottom-level noise level $\trace{\Sigma_\xi}$ and the aggregation structure.
For a binary summing matrix, $\norm{\mat S}{\mathrm{F}}^2$ counts the total number of bottom-series memberships across the hierarchy. Indeed,
$\norm{\mat S}{\mathrm{F}}^2=\sum_{i=1}^{d}\sum_{j=1}^{d_b}S_{ij}^2 =\sum_{i=1}^{d}\sum_{j=1}^{d_b}S_{ij}$, and
$S_{ij}=1$ exactly when the bottom-level series $j$ contributes to series $i$.
Consequently, a deeper or more densely aggregated hierarchy can increase this upper bound on the propagated noise.
\end{remark}

%% file: 6-Experiments.tex
\section{Numerical Experiments}
\label{sec:exper}

In this section, we present an empirical study of the proposed predictors on both synthetic and real-world datasets. For the two-stage baseline models, the reconciliation methods are implemented using the \textit{HierarchicalForecast} library\footnote{\url{https://pypi.org/project/hierarchicalforecast/}} \citep{olivares_hierarchicalforecast_2022}. For evaluation, we report the mean squared error (MSE) computed at each hierarchical level as well as across all levels. 

\subsection{Synthetic Dataset}
\label{sec:toy_experiment}
\paragraph{Generation Process.} We generate a three-level hierarchy with $d=8$ series: one top-level, two middle-level, and $d_b=5$ bottom-level series. Each bottom-level series follows an additive trend--seasonality model with correlated Gaussian noise; series sharing a parent have the same trend and seasonal components. The remaining series are obtained by aggregation with $\mat S$, ensuring coherence by construction. A scalar $\alpha\geq0$ controls the noise level. Exact parameter values, the covariance construction, and an example realization are provided in \Cref{app:toy_hts}.

\paragraph{Experimental Setting.} We investigate how model performance depends on the training length $T$, with the covariance scaling parameter fixed at $\alpha=0.6$. We take $T\in\set{40,80,160,320}$ and run 32 Monte Carlo simulations for each value. In each simulation, we generate hierarchical time series of length $T+200$: the first $T$ timestamps are used for training, the next $100$ for validation and hyperparameter selection, and the last $100$ for testing. For the input covariates, we use $x_t = (t, \sin(\nicefrac{2\pi t}{7}), \sin(\nicefrac{2\pi t}{30})) \in \reals^3$ for all series in the hierarchy. We compare KRR followed by reconciliation with our RVVKRR predictor. For KRR, the ridge parameter $\lambda$ is cross-validated separately for each series, and we consider some reconciliation strategies described in \Cref{app:reconciliation}. For RVVKRR, we use the two separable OVKs introduced in \Cref{example:toy_laplacian_kernel}: the identity kernel (RVVKRR-I) and the graph-based kernel (RVVKRR-Lap). The graph connects bottom-level series sharing the same parent, and the corresponding edge weights are selected by validation. The node-wise diagonal weights are set to the KRR ridge parameters.

\paragraph{Experimental Results.}
\begin{figure}[htbp!]
    \centering
    \includegraphics[width=0.8\linewidth]{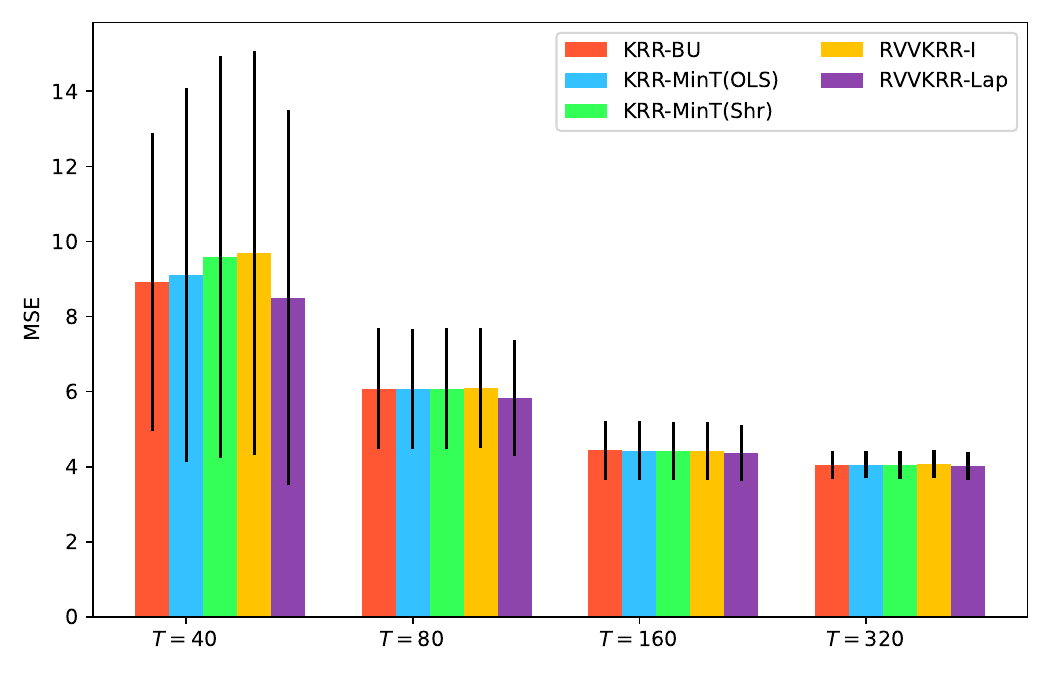}
    \caption{All-level MSE on the synthetic dataset as a function of the training length $T$, with $\alpha=0.6$.}
    \label{fig:toy_res}
\end{figure}

\Cref{fig:toy_res} illustrates the impact of the training length $T$ on the MSE across all levels. Employing a Laplacian OVK yields better results than using an identity kernel. Moreover, RVVKRR consistently outperforms the two-stage methods, particularly in low-data regimes ($T=40$ or $T=80$). Additional results on the effect of the covariance scaling parameter $\alpha$ are reported in \Cref{app:toy_hts}.

\subsection{TourismLarge Dataset}
\label{sec:tourism_experiment}

\paragraph{Dataset.}
We next consider the Australian domestic TourismLarge dataset \citep{wickramasuriya_optimal_2019}, which contains 228 monthly observations from January 1998 to December 2016.  The hierarchy comprises $d=415$ series: one country, 7 states, 27 zones, 76 regions, and 304 bottom-level region--purpose series.  The four bottom-level purposes are holiday, visiting friends and relatives (VFR), business, and other.

\paragraph{Experimental Setting.}
We use the preceding 24 observations of each series as its covariates.  After constructing these lags, 204 observations remain; the last 24 form the test set and the first 180 form the training set. Hyperparameters are selected by three-fold expanding-window cross-validation. Independent OLS and KRR each use only the 24 lags of the series being predicted. KRR uses a linear input kernel. For both independent estimators, we report unreconciled forecasts and forecasts reconciled using the reconciliation methods described in \Cref{app:reconciliation}. RVVKRR receives the concatenation of all series-specific lag blocks, hence $x_t \in\reals^{24d}$. The OVK used by RVVKRR combines the non-separable base kernel of \Cref{example:nonsep_base_kernel} with the MinT (shrinkage) projection of \Cref{example:gen_ovk_mint}. The scalar-valued kernel in \Cref{eq:nonsep_base_kernel} is the linear kernel. We take $\mat B=\mat B_{\mathrm{pur}}\in\reals^{d\times 5}$, whose $r$-th row is the one-hot encoding of the membership of series $r$ in either the aggregate category or one of the four travel-purpose categories. 
These relationships are not explicitly encoded by the aggregation matrix. The parameters $\lambda_r$ are inherited from the independent KRR. Cross-validation selects $\rho=0.1$.

\paragraph{Experimental Results.}
\begin{table}[htbp!]
    \centering
    \footnotesize
    \setlength{\tabcolsep}{1.8pt}
    \caption{Test MSE for KRR and RVVKRR at each aggregation level of TourismLarge, reported in millions ($10^6$). All sums the level-wise MSEs. $\dagger$ denotes unreconciled forecasts. Lower is better.}
    \begin{tabular}{lcccccc}
        \toprule
        Method & Ctry. & State & Zone & Reg. & Purp. & All \\
        \midrule
        \multicolumn{7}{@{}l}{\textit{KRR base forecasts}} \\
        Unreconciled $\dagger$ & 2.71 & 1.69 & 1.30 & 1.19 & 1.13 & 8.02 \\
        Bottom-Up & 8.06 & 2.53 & 1.53 & 1.21 & 1.13 & 14.46 \\
        Top-Down (AHP) & 2.71 & 3.94 & 2.38 & 2.05 & 1.92 & 13.00 \\
        Top-Down (PHA) & 2.71 & 3.89 & 2.37 & 2.04 & 1.92 & 12.94 \\
        MinT (OLS) & 2.72 & 1.58 & 1.25 & 1.12 & 1.11 & 7.78 \\
        MinT (shrinkage) & 3.05 & 1.60 & \textbf{1.21} & \textbf{1.08} & \textbf{1.06} & 7.99 \\
        MinT (WLS-var) & 4.56 & 1.83 & 1.29 & 1.11 & 1.09 & 9.89 \\
        ERM (closed form) & 5.36 & 2.50 & 2.61 & 2.57 & 2.52 & 15.57 \\
        \midrule
        RVVKRR & \textbf{2.02} & \textbf{1.52} & 1.30 & 1.24 & 1.57 & \textbf{7.64} \\
        \bottomrule
    \end{tabular}
    \label{tab:tourism_results}
\end{table}
\Cref{tab:tourism_results} reports the test MSE at each aggregation level; the full comparison, including OLS base forecasts and alternative RVVKRR projections, is given in \Cref{app:tourism}. Relative to the matched two-stage KRR with MinT (shrinkage) baseline, the RVVKRR predictor with the  CV-selected purpose prior ($\rho=0.1$) reduces the all-level MSE from $7.99\times10^6$ to $7.64\times10^6$. The gains are concentrated at the country and state levels, with larger errors at the lower levels.

\subsection{Traffic Dataset}
\label{sec:traffic_experiment}

\paragraph{Dataset.}
We finally consider the SF Bay Area daily Traffic dataset, originally preprocessed by \citet{ben_taieb_regularized_2019} and distributed with the \textit{HierarchicalForecast} benchmark \citep{olivares_hierarchicalforecast_2022}. It contains 366 daily observations from January~1 to December~31, 2008, for $d=207$ series arranged in four levels: one total series, two Level~2 aggregates, four Level~3 aggregates, and $d_b=200$ bottom-level series.

\paragraph{Experimental Setting.}
For each series $r$, the covariate vector $x_{t,r}\in\reals^{23}$ contains 16 raw lag values at lags $1,\ldots,14$, 21, and 28, six Fourier features corresponding to harmonics $1$, $2$, and $3$ of the weekly period $P=7$, and a linear trend. As for TourismLarge, independent KRR uses only $x_{t,r}$ to predict series $r$, whereas RVVKRR receives the concatenation $x_t=(x_{t,1},\ldots,x_{t,d})\in\reals^{23d}$. No feature standardization is applied. Independent KRR uses a linear kernel, with a separate ridge parameter $\lambda_r$ selected for each series by time-series cross-validation. For RVVKRR, we set $\rho=0$ in the base OVK of \Cref{example:nonsep_base_kernel} and compose this kernel with a MinT projection as in \Cref{example:gen_ovk_mint}. The parameters $\lambda_r$ are inherited from the fitted independent KRR and set the relative regularization scales. We consider three projections inside the kernel: MinT (OLS), MinT (shrinkage), and MinT (WLS-var). Within each training window, the covariance proxies for the latter two are estimated from in-sample residuals of auxiliary independent KRR predictors. For each RVVKRR variant, the global regularization parameter $\lambda$ is selected by cross-validation. The corresponding KRR+MinT baselines apply the same OLS, shrinkage, or variance-weighted projection to independently fitted KRR forecasts post hoc. The full comparison with additional models and further experimental details are provided in \Cref{app:traffic}.

\paragraph{Experimental Results.}
\Cref{tab:traffic_results} shows that post-hoc MinT reconciliation does not necessarily improve independent KRR: OLS and shrinkage yield only marginal changes, while WLS-var degrades performance. In contrast, placing the shrinkage or WLS-var projection inside the joint RVVKRR fit substantially reduces the all-level MSE. This benefit is not due to coherence alone, since RVVKRR-MinT (OLS) remains close to independent KRR. Unlike post-hoc reconciliation, which only transforms already fitted forecasts, an internal projection changes the RKHS geometry and allows the series-specific feature blocks to be estimated jointly. 
The resulting gains occur mainly at the aggregate levels, at the cost of some bottom-level accuracy under the stated all-level MSE objective.

\begin{table}[htbp!]
    \centering
    \footnotesize
    \setlength{\tabcolsep}{1.8pt}
    \caption{Test MSE at each aggregation level of Traffic, reported in thousands ($10^3$). All sums the level-wise MSEs. $\dagger$ denotes incoherent forecasts. Lower is better.}
    \begin{tabular}{lccccc}
        \toprule
        Method & L1 & L2 & L3 & Bottom & All \\
        \midrule
        \multicolumn{6}{@{}l}{\textit{KRR base forecasts}} \\
        Unreconciled $\dagger$ & 8.95 & 5.05 & 2.84 & 0.63 & 17.47 \\
        MinT (OLS) & 9.01 & 4.95 & 2.85 & 0.61 & 17.42 \\
        MinT (shrinkage) & 9.16 & 4.83 & 2.69 & \textbf{0.60} & 17.28 \\
        MinT (WLS-var) & 11.68 & 6.11 & 3.34 & 0.62 & 21.76 \\
        \midrule
        \multicolumn{6}{@{}l}{\textit{RVVKRR forecasts}} \\
        MinT (OLS) & 8.99 & 4.93 & 2.79 & 0.61 & 17.32 \\
        MinT (shrinkage) & 3.84 & 2.14 & 2.20 & 3.39 & 11.57 \\
        MinT (WLS-var) & \textbf{0.32} & \textbf{0.24} & \textbf{0.44} & 3.02 & \textbf{4.03} \\
        \bottomrule
    \end{tabular}
    \label{tab:traffic_results}
\end{table}

%% file: 7-Conclusion.tex
\section{Conclusion}
\label{sec:conclusion}
We introduced RVVKRR, an end-to-end kernel framework for hierarchical time series forecasting that enforces coherence directly in the hypothesis space.
Experiments on synthetic and real datasets show that embedding an appropriate hierarchical geometry within joint estimation can outperform post-hoc reconciliation, while coherence alone does not guarantee an improvement.
Future work should learn the coherent kernel geometry from data, accommodate level-specific forecasting priorities, scale to larger hierarchies, and extend the analysis to dependent time-series observations.

%% file: 9-Appendix.tex
\section{Additional Preliminaries}
\subsection{Hierarchical Time Series Example}
\label{app:hts_example}

Example \ref{example: hts} gives an illustration of such a hierarchical time series.

\begin{figure}[htbp!]
    \centering
    \includegraphics[width=0.7\linewidth]{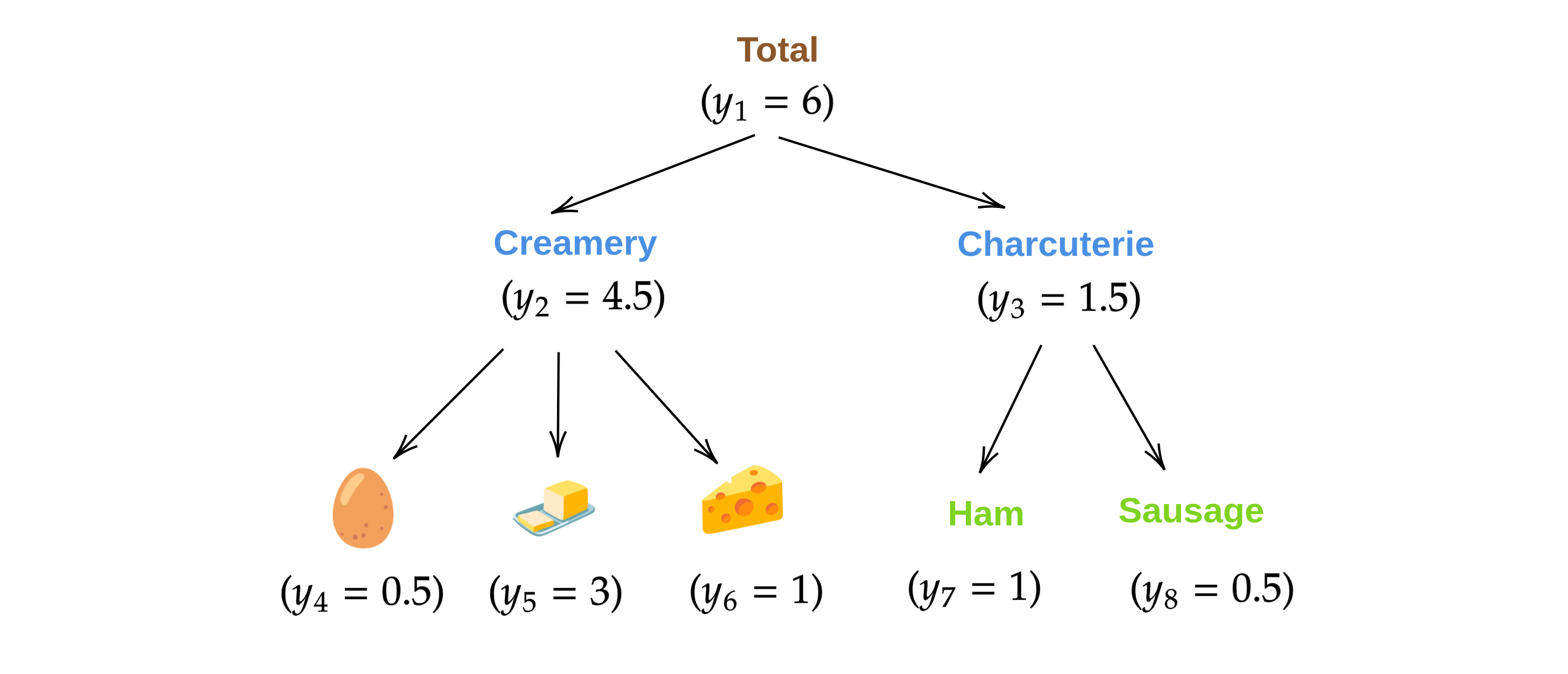}
    \caption{A three level hierarchical time series example for the daily raw material demand of a crêperie.}
    \label{fig:hts_example}
\end{figure}

\begin{example}
  Imagine that we have 5 univariate time series representing the daily raw material demand of a crêperie, for example: eggs, butter, cheese, sausage, and ham. The demand for these products can be aggregated to form two new univariate time series representing, respectively, the charcuterie and creamery categories. These two time series can be further aggregated to obtain the total demand. The hierarchical structure is shown in Figure \ref{fig:hts_example}. In this example, the summation matrix is defined as follows:
  \begin{equation}
    \mat S_{\mathrm{sum}} =  \begin{bmatrix} 1 & 1 & 1 & 1 & 1 \\ 1 & 1 & 1 & 0 & 0 \\ 0 & 0 & 0 & 1 & 1\end{bmatrix}.
  \end{equation}

\label{example: hts}
\end{example}

\subsection{Reconciliation Strategies}
\label{app:reconciliation}

As explained in \Cref{bg:hts}, a linear reconciliation method transforms a base
forecast $\tilde{y}_t\in\reals^d$ into the coherent forecast
\begin{equation*}
  \hat{y}_t = \mat S\mat G\tilde{y}_t \in \im(\mat S),
\end{equation*}
where $\mat G\in\reals^{d_b\times d}$ maps forecasts from all levels of the
hierarchy to the bottom level. The main classical strategies differ in how
they choose $\mat G$.

\paragraph{Bottom-up reconciliation.}
The bottom-up approach retains the bottom-level base forecasts and discards
the forecasts at the aggregate levels. Its reconciliation matrix is
\begin{equation*}
  \mat G_{\mathrm{BU}}
  = \begin{bmatrix}
      0_{d_b\times d_a} & \idmat{d_b}
    \end{bmatrix}.
\end{equation*}
Thus, $\mat G_{\mathrm{BU}}\tilde{y}_t$ selects the bottom-level forecasts,
which are then aggregated by $\mat S$ to recover forecasts at every level of
the hierarchy.

\paragraph{Top-down reconciliation.}
The top-down approach starts from the base forecast of the top-level series
and disaggregates it across the bottom-level series \citep{gross_disaggregation_1990}.
It uses
\begin{equation*}
  \mat G_{\mathrm{TD}}
  = \begin{bmatrix}
      p & 0_{d_b\times(d-1)}
    \end{bmatrix},
  \qquad p\in\reals^{d_b},
\end{equation*}
where $p$ contains the disaggregation proportions. Let
$y_t^{\mathrm{top}}\in\reals$ denote the observed top-level series and let
$y_{t,i}^b$ denote the $i$-th bottom-level series, for $t\in\intset{T}$ and
$i\in\intset{d_b}$. The \textit{average historical proportions} (AHP)
method averages the share of each bottom-level series at every training time:
\begin{equation*}
  p_i^{\mathrm{AHP}}
  = \frac{1}{T}\sum_{t=1}^T
    \frac{y_{t,i}^b}{y_t^{\mathrm{top}}},
  \qquad i\in\intset{d_b}.
\end{equation*}
By contrast, the \textit{proportions of the historical averages} (PHA)
method computes the share from the historical mean of each series:
\begin{equation*}
  p_i^{\mathrm{PHA}}
  = \frac{T^{-1}\sum_{t=1}^T y_{t,i}^b}
         {T^{-1}\sum_{t=1}^T y_t^{\mathrm{top}}}
  = \frac{\sum_{t=1}^T y_{t,i}^b}
         {\sum_{t=1}^T y_t^{\mathrm{top}}},
  \qquad i\in\intset{d_b}.
\end{equation*}
For coherent training data, both choices satisfy
$\sum_{i=1}^{d_b}p_i=1$. The vector
$p=(p_1,\ldots,p_{d_b})^{\trans}$ obtained with either method is then used in
$\mat G_{\mathrm{TD}}$ above.

\paragraph{MinT reconciliation.}
Let $\mat W\in\mathbb S_{++}^d$ be a positive-definite covariance proxy for the
base-forecast errors. The minimum trace (MinT) reconciliation matrix and its
associated coherent projection are \citep{wickramasuriya_optimal_2019}
\begin{equation*}
  \mat G_{\mat W}
  = (\mat S^{\trans}\mat W^{-1}\mat S)^{-1}
    \mat S^{\trans}\mat W^{-1},
  \qquad
  \mat P_{\mat W}=\mat S\mat G_{\mat W}.
\end{equation*}
Ideally, $\mat W$ is the covariance matrix
$\cov(y_t-\tilde{y}_t)$, but this quantity is unknown and must be estimated.
Let
\begin{equation*}
  \widehat{\mat W}
  = \frac{1}{T}\sum_{t=1}^T
    (y_t-\tilde{y}_t)(y_t-\tilde{y}_t)^{\trans}
\end{equation*}
denote the empirical covariance of the base-forecast errors over the $T$
training observations. The variants used in our experiments correspond to
the following choices of the covariance proxy.

\textit{OLS.}
The OLS reconciliation of \citet{hyndman_optimal_2011} is the special case of
MinT obtained with
\begin{equation*}
  \mat W_{\mathrm{OLS}}=\mat I_d.
\end{equation*}
Indeed,
\begin{equation*}
  \mat G_{\mat W_{\mathrm{OLS}}}
  = (\mat S^{\trans}\mat S)^{-1}\mat S^{\trans},
  \qquad
  \mat P_{\mat W_{\mathrm{OLS}}}
  = \mat S(\mat S^{\trans}\mat S)^{-1}\mat S^{\trans},
\end{equation*}
so $\mat P_{\mat W_{\mathrm{OLS}}}$ is the orthogonal projector onto
$\im(\mat S)$. 

\textit{WLS-variance.}
The weighted least squares variance variant retains only the marginal
base-forecast error variances:
\begin{equation*}
  \mat W_{\mathrm{WLS}}
  = \diag(\widehat{\mat W})
  = \diag(\widehat{\sigma}_1^2,\ldots,\widehat{\sigma}_d^2),
\end{equation*}
where $\widehat{\sigma}_r^2$ is the empirical error variance of series $r$.
This diagonal proxy ignores cross-series error correlations.

\textit{Shrinkage.}
The shrinkage variant combines the full empirical covariance matrix with its
diagonal counterpart:
\begin{equation*}
  \mat W_{\mathrm{shr}}
  = (1-\alpha)\widehat{\mat W}
    + \alpha\diag(\widehat{\mat W}),
  \qquad \alpha\in(0,1].
\end{equation*}
It therefore preserves estimated cross-series error correlations while
shrinking them toward zero. For a comprehensive review of classical
reconciliation strategies, see \citet{hyndman_forecasting_2018} and the
references therein.

\subsection{Separable Operator-Valued Kernels}
\label{app:ovk_sep}
Separable OVKs provides a clear decoupling between the contribution of input and output, where $\kernel$ measures input similarity and $\op A$ encodes the relationships between output components, which makes them highly interpretable compared to general OVKs. Moreover, they simplify massively the computation burden for some kernel-based learning algorithms such as multi-task kernel regression \citep{micchelli_kernels_2004} or multi-task Gaussian processes \citep{bonilla_multi-task_2007}. See, e.g., \citet[Section 6]{alvarez_kernels_2012} for the discussion on the computation consideration.

\subsection{Operator-Valued Kernels in Supervised Learning}
\label{app:ovk_sup}
In multi-task learning, for instance, it is natural to take the output space to be $\reals^d$ where $d$ denotes the number of related tasks whose joint learning may yield mutual benefits \citep{micchelli_kernels_2004, evgeniou_learning_2005}. More recently, several works have focused on predicting functional data by setting  $\mathcal{Y}=L^2(\Omega)$, the set of square integrable functions on a given compact set $\Omega \subset \reals^d$ \citep{kadri_operatorvalued_2016, bouche_nonlinear_2021, bouche_functionvalued_2023}. In the context of structured prediction, complex output objects can be embedded into an RKHS $\mathcal H_{k}$ using a scalar-valued kernel $k: \mathcal{Y} \times \mathcal{Y} \to \reals$, so that the learning problem can be performed by finding a function in an $\mathcal H_{k}$-valued RKHS \citep{brouard_semisupervised_2011, kadri_generalized_2013, brouard_fast_2016}. 

Moreover, learning functions of the form $f = \op P \circ h $ where $\op P: \mathcal{W} \to \mathcal{Y}$ is a bounded linear operator and $h: \mathcal{X} \to \mathcal{W}$ is a function from a $\mathcal{W}$-valued RKHS, is not totally new in the kernel literature. In vector-valued multi-view learning, \citet{minh_unifying_2016} consider $\op P$ as the combination operator that transforms the output from the $m$ view in  $\mathcal{W}= \mathcal{Y}^m$ into a single output in $\mathcal{Y}$. In \citet{bouche_nonlinear_2021}, the operator $\op P$ is a projection operator that transforms the expansion coefficient $u \in \mathcal{W} = \reals^d$ to a function $\sum_{l=1}^{d} u_l \phi_l \in \mathcal{Y} = L^2(\Omega)$ where $(\phi_l)_{l=1}^{d} \in L^2(\Omega)^d$ is a dictionary of functions.

\section{Additional Details for Section~\ref{sec:predictor}}
\label{app:add_details_sec_predictor}
\subsection{Dimension of the Coherent Forecast Space}

\begin{lemma} \label{lemma:dim_im_S}
  Let $\mat S$ be the aggregation matrix defined in Equation \ref{eq:aggre_matrix}. Then $\dim( \im(\mat S)) = d_b$, where $d_b$ is the number of bottom-level series.
\end{lemma}
\begin{proof}
  Suppose that we have $\mat S \vect v = 0$ for some $\vect v \in \reals^{d_b}$. By definition of the matrix $\mat S$, we have
  \begin{equation*}
    \begin{bmatrix} \mat S_{\mathrm{sum}} \\ \idmat{d_b} \end{bmatrix} \vect v = \begin{bmatrix} \mat S_{\mathrm{sum}} \vect v \\ \vect v \end{bmatrix} = 0.
  \end{equation*}
  The lower block gives $\vect v= 0$. So the only solution of $\mat S \vect v = 0$ is $\vect v= 0$. Thus, we have $\kermap (\mat S) = \set{0}$ and $\dim(\kermap (\mat S))=0$. By the rank–nullity theorem, we can conclude that $\dim (\im(\mat S)) = \dim (\reals^{d_b}) - \dim(\kermap (\mat S)) = d_b.$
\end{proof}

\subsection{Proof of \Cref{lemma:ovk_image_s}}
\label{sec:proof_lemma_ovk_image_s}
\begin{proof}
  Let $\mat A \in \lineop{\im(\mat S)}$ be self-adjoint and positive, and set $\mat P \coloneqq \mat{U}_{\mat S}\mat{U}_{\mat S}^{\trans}$, the orthogonal projection onto $\im(\mat S)$. Define
  \begin{equation*}
    \mat A' \coloneqq \mat{U}_{\mat S}^{\trans}\mat A\mat{U}_{\mat S} \in \reals^{d_b\times d_b}.
  \end{equation*}
  This composition is well-defined since $\mat{U}_{\mat S}$ maps $\reals^{d_b}$ into $\im(\mat S)$. For every $y \in \im(\mat S)$, $\mat A y$ belongs to $\im(\mat S)$, so $\mat P y=y$ and $\mat P\mat A y=\mat A y$. Consequently,
  \begin{equation*}
    \mat{U}_{\mat S}\mat A'\mat{U}_{\mat S}^{\trans}y
    = \mat P\mat A\mat P y
    = \mat A y,
    \qquad \forall y \in \im(\mat S).
  \end{equation*}
  Thus, $\mat{U}_{\mat S}\mat A'\mat{U}_{\mat S}^{\trans}$ is a matrix representation of $\mat A$.

  For any $\vect v,\vect w \in \reals^{d_b}$, self-adjointness of $\mat A$ gives
  \begin{align*}
    \inner{\mat A'\vect v}{\vect w}{\reals^{d_b}}
    = \inner{\mat A\mat{U}_{\mat S}\vect v}{\mat{U}_{\mat S}\vect w}{\reals^d}
    = \inner{\mat{U}_{\mat S}\vect v}{\mat A\mat{U}_{\mat S}\vect w}{\reals^d}
    = \inner{\vect v}{\mat A'\vect w}{\reals^{d_b}}.
  \end{align*}
  Hence $\mat A'$ is symmetric. Positivity of $\mat A$ also yields
  \begin{equation*}
    \inner{\mat A'\vect v}{\vect v}{\reals^{d_b}}
    = \inner{\mat A\mat{U}_{\mat S}\vect v}{\mat{U}_{\mat S}\vect v}{\reals^d}
    \geq 0,
    \qquad \forall \vect v \in \reals^{d_b},
  \end{equation*}
  so $\mat A'$ is positive semi-definite.
\end{proof}

\subsection{Diagonal-based Separable OVKs}
\label{sec:ovk_sep_diag}
A specific choice of the separable OVK is to consider a diagonal matrix as $A'$, i.e., $\ovk(x, x') = \kernel(x, x') \mat{U}_{\mat S} \mat A' \mat{U}_{\mat S}^{\trans}$ with
\begin{equation} \label{eq:task_kernel_diag}
  \mat A' = \diag\left(\lambda_1, \dots, \lambda_{d_b}\right)
\end{equation}
where $\lambda_i > 0$ for all $i \in \intset{d_b}$. We will show here actually this predictor corresponds to the one of a learning algorithm based on $d_b$ independent predictors of kernel ridge regression \citep[KRR,][]{saunders_ridge_1998, evgeniou_regularization_2000}.

When $\mat A'$ is defined as in \Cref{eq:task_kernel_diag}, solving \Cref{eq:ls_vvkrr_sep_v2} equals to solving the following linear system
\begin{equation} \label{eq:ls_krr_indep}
  \frac{T \lambda}{\lambda_i} e_i + \mat K e_i  = \left(\mat Y^{\trans} \mat{U}_{\mat S}\right)_{:,i}  \quad \text{with} \; e_i \in \reals^{T}
\end{equation}
for all $i \in \intset{d_b}$, and then defining $\mat E$ such that $\mat E_{i, :} = \frac{e_i^{\trans}}{\lambda_i}$. We observe that \Cref{eq:ls_krr_indep} actually corresponds to the linear system arising in KRR \citep[see, e.g.,][Equation (4.3)]{evgeniou_regularization_2000}, where the regularization parameter is $\frac{\lambda}{\lambda_i}$ and the training label corresponds to the $i$-th component of the transformed output $y' = \mat{U}_{\mat S}^{\trans} y \in \reals^{d_b}$. 

For a new input $x \in \mathcal X$, let $\vect k_x = (\kernel(x,x_1),\dots,\kernel(x,x_T))^{\trans}$ and denote the corresponding scalar KRR predictions by $\hat g_i(x) = e_i^{\trans}\vect k_x$. Since the $i$-th row of $\mat A'\mat E$ is $e_i^{\trans}$, the RVVKRR prediction is
\begin{equation}
  \hat f(x) = \mat{U}_{\mat S}\mat A'\mat E\vect k_x
  = \mat{U}_{\mat S}
  \begin{bmatrix}
    \hat g_1(x) \\ \vdots \\ \hat g_{d_b}(x)
  \end{bmatrix}.
\end{equation}
Thus, the predictor is obtained by fitting $d_b$ independent KRR predictors to the transformed output coordinates and mapping their predictions back to $\im(\mat S)$ through $\mat{U}_{\mat S}$.

\subsection{Proof of \Cref{lemma:nearest_task_kernel}}
\label{sec:proof_lemma_nearest_task_kernel}
\begin{proof}
  We denote the objective function by $J(\mat A') = \norm{\mat M - \mat{U}_{\mat S} \mat A' \mat{U}_{\mat S}^{\trans}}{\textrm{F}}^2$.
  Since the objective function is continuous and coercive, and the PSD cone $\mathbb{S}^{d_b}_+$ is closed, there exists a minimizer. Moreover, since $J$ is strictly convex, the minimizer is unique. Actually, the minimization problem is a so-called convex quadratic SDP problem. Due to the fact that
  \begin{align}
    J(\mat A') & = \norm{\mat{U}_{\mat S} \mat A' \mat{U}_{\mat S}^{\trans}}{\textrm{F}}^2 - 2 \inner{\mat M}{\mat{U}_{\mat S} \mat A' \mat{U}_{\mat S}^{\trans}}{\textrm{F}} + \norm{\mat M}{\textrm{F}}^2 \\
    & = \norm{\mat A'}{\textrm{F}}^2 - 2 \inner{\mat{U}_{\mat S}^{\trans}\mat M \mat{U}_{\mat S} }{\mat A' }{\textrm{F}} + \norm{\mat M}{\textrm{F}}^2,
    \label{eq:convex_quad_sdp}
  \end{align}
  we have
  \begin{align}
    \argmin_{\mat A' \in \mathbb{S}^{d_b}_+} J(\mat A')  & = \argmin_{\mat A' \in \mathbb{S}^{d_b}_+} \norm{\mat A'}{\textrm{F}}^2 - 2 \inner{\mat{U}_{\mat S}^{\trans}\mat M \mat{U}_{\mat S} }{\mat A' }{\textrm{F}} + \norm{\mat{U}_{\mat S}^{\trans}\mat M \mat{U}_{\mat S}}{\textrm{F}}^2 \\
    & = \argmin_{\mat A' \in \mathbb{S}^{d_b}_+} \norm{\mat A' - \mat{U}_{\mat S}^{\trans}\mat M \mat{U}_{\mat S}}{\textrm{F}}^2
  \end{align}
  whose solution is the nearest PSD matrix of $\mat{U}_{\mat S}^{\trans}\mat M \mat{U}_{\mat S}$ in Frobenius norm. \citet{higham_computing_1988} shows that the solution is obtained by eigenvalue decomposition of $\mat{U}_{\mat S}^{\trans}\mat M \mat{U}_{\mat S}$ and zeroing out negative eigenvalues, thus concluding the proof of the lemma.
\end{proof}

\subsection{Proof of \Cref{corol:nearest_task_kernel}}
\label{app:proof_corol_nearest_task_kernel}
\begin{proof}
  By \Cref{lemma:nearest_task_kernel}, we have $\argmin_{\mat A' \in \mathbb{S}^{d_b}_+} \norm{\mat M - \mat{U}_{\mat S} \mat A' \mat{U}_{\mat S}^{\trans}}{\textrm{F}}^2 = \argmin_{\mat A' \in \mathbb{S}^{d_b}_+} \norm{\mat A' - \mat{U}_{\mat S}^{\trans}\mat M \mat{U}_{\mat S}}{\textrm{F}}^2$.
  Since $\mat M \in \mathbb{S}^{d}_+$, it follows that $\mat{U}_{\mat S}^{\trans}\mat M \mat{U}_{\mat S} \in \mathbb{S}^{d_b}_+$. Consequently, $\mat A'_{\mathrm{opt}} = \mat{U}_{\mat S}^{\trans}\mat M \mat{U}_{\mat S} \in \mathbb{S}^{d_b}_+$. It then follows that 
  \begin{equation}
    \mat A_{\mathrm{opt}} = \mat{U}_{\mat S} \mat A'_{\mathrm{opt}} \mat{U}_{\mat S}^{\trans} = \mat{U}_{\mat S} \mat{U}_{\mat S}^{\trans}\mat M \mat{U}_{\mat S} \mat{U}_{\mat S}^{\trans} = \mat P \mat M \mat P 
  \end{equation}
  where $\mat P \coloneqq \mat{U}_{\mat S} \mat{U}_{\mat S}^{\trans}$ is actually the orthogonal projection onto $\im(\mat S)$.
\end{proof}

\subsection{Existence of Feature Map for the Non-separable OVK Defined in \Cref{example:nonsep_base_kernel}}
\label{app:proof_example_nonsep_base_kernel_feature_map}

\begin{proposition}
  $\widetilde{\ovk}_\rho$ defined in \Cref{example:nonsep_base_kernel} is a valid OVK since it admits a feature operator $\Gamma_\rho$ such that for all $x,x'\in\mathcal X$,
\begin{equation}
  \widetilde{\ovk}_\rho(x,x')=\Gamma_\rho(x)\Gamma_\rho(x')^{\adj},
\end{equation}
where the feature operator $\Gamma_\rho(x):\reals^{q_\rho\times p_\phi}\to\reals^d$ is defined, for all $r\in\intset{d}$, by
\begin{equation}
  [\Gamma_\rho(x)\Theta]_r
  =
  \frac{1}
       {\sqrt{\lambda_r}} \left\langle
    \phi(x_r),\Theta^{\trans}(\mat L_\rho)_{r,:}^{\trans}
  \right\rangle,
  \end{equation} and its adjoint is
\begin{equation}
  \Gamma_\rho(x)^{\adj}v
  =
  \sum_{r=1}^d
  \frac{v_r}{\sqrt{\lambda_r}}
  (\mat L_\rho)_{r,:}^{\trans}\phi(x_r)^{\trans}
  \in\reals^{q_\rho\times p_\phi},
\end{equation}
where $\mat L_\rho\in\reals^{d\times q_\rho}$ is any factor satisfying
$\mat A_\rho=\mat L_\rho\mat L_\rho^{\trans}$.
\end{proposition}

\begin{proof}
  First, $\mat A_\rho$ is positive semidefinite. Indeed, for every
  $v\in\reals^d$,
  \begin{equation}
    v^{\trans}\mat A_\rho v
    =(1-\rho)\norm{v}{\reals^d}^2
      +\rho\norm{\mat B^{\trans}v}{\reals^q}^2
    \geq 0.
  \end{equation}
  It therefore admits a factorization
  $\mat A_\rho=\mat L_\rho\mat L_\rho^{\trans}$, where we may take
  $q_\rho=\rank(\mat A_\rho)$ and
  $\mat L_\rho\in\reals^{d\times q_\rho}$.

  Equip the feature space
  $\mathcal F_\rho=\reals^{q_\rho\times p_\phi}$ with the Frobenius inner
  product.  For $x\in\mathcal X$, define the linear operator
  $\Gamma_\rho(x):\mathcal F_\rho\to\reals^d$ as 
  \begin{equation}
  [\Gamma_\rho(x)\Theta]_r
  =
  \frac{1}
       {\sqrt{\lambda_r}} \left\langle
    \phi(x_r),\Theta^{\trans}(\mat L_\rho)_{r,:}^{\trans}
  \right\rangle.
  \label{eq:nonsep_feature_operator}
\end{equation}
For every
  $\Theta\in\mathcal F_\rho$ and $v\in\reals^d$, we have
  \begin{align}
    \inner{\Gamma_\rho(x)\Theta}{v}{\reals^d}
    &=\sum_{r=1}^d
      \frac{v_r}{\sqrt{\lambda_r}}
      \inner{\phi(x_r)}
            {\Theta^{\trans}(\mat L_\rho)_{r,:}^{\trans}}
            {\reals^{p_\phi}}\\
    &=\inner{\Theta}{
      \sum_{r=1}^d\frac{v_r}{\sqrt{\lambda_r}}
      (\mat L_\rho)_{r,:}^{\trans}\phi(x_r)^{\trans}}
      {\mathcal F_\rho}.
  \end{align}
  Hence its adjoint is precisely the following operator:
  \begin{equation}
  \Gamma_\rho(x)^{\adj}v
  =
  \sum_{r=1}^d
  \frac{v_r}{\sqrt{\lambda_r}}
  (\mat L_\rho)_{r,:}^{\trans}\phi(x_r)^{\trans}
  \in\reals^{q_\rho\times p_\phi}.
  \label{eq:nonsep_feature_adjoint}
\end{equation}

  It remains to verify the factorization of the kernel.  For any
  $x,x'\in\mathcal X$, $v\in\reals^d$, and $r\in\intset{d}$,
  \begin{align}
    [\Gamma_\rho(x)\Gamma_\rho(x')^{\adj}v]_r
    &=\frac{1}{\sqrt{\lambda_r}}
      \inner{\phi(x_r)}{
        \left(
        \sum_{s=1}^d\frac{v_s}{\sqrt{\lambda_s}}
        (\mat L_\rho)_{s,:}^{\trans}\phi(x'_s)^{\trans}
        \right)^{\trans}
        (\mat L_\rho)_{r,:}^{\trans}}
        {\reals^{p_\phi}}\\
    &=\sum_{s=1}^d
      \frac{(\mat L_\rho)_{r,:}(\mat L_\rho)_{s,:}^{\trans}}
           {\sqrt{\lambda_r\lambda_s}}
      \inner{\phi(x_r)}{\phi(x'_s)}{\reals^{p_\phi}}v_s\\
    &=\sum_{s=1}^d
      \frac{(\mat A_\rho)_{rs}}
           {\sqrt{\lambda_r\lambda_s}}
      \kernel(x_r,x'_s)v_s
     =[\widetilde{\ovk}_\rho(x,x')v]_r.
  \end{align}
  Since this identity holds for every $v\in\reals^d$, it follows that
  \begin{equation}
    \widetilde{\ovk}_\rho(x,x')
    =\Gamma_\rho(x)\Gamma_\rho(x')^{\adj}.
  \end{equation}
  Thus $\widetilde{\ovk}_\rho$ is a valid operator-valued kernel, and
  $\Gamma_\rho$ is one of its feature operators.
\end{proof}

\subsection{Proof of the Equivalence of the Two Learning Formulations}
\label{app:proof_prop_same_estimator_sg}

The following lemma summarizes the relationship between the RKHS associated to $\ovk_{\mat S \mat G}$ and the one associated to $\widetilde{\ovk}$.
\begin{lemma} \label{lemma:compose_ovk}
  Let $\widetilde{\ovk}$ be a $\lineop{\reals^{d}}$-valued kernel, and let $\mathcal H_{\widetilde{\ovk}}$ be its associated RKHS. Let $\mat S\mat G: \reals^d \to \reals^d$ be a bounded operator. Define
  \begin{equation}
    \mathcal{H}_{\widetilde{\ovk}, \mat S\mat G} \coloneqq \set{\mat S\mat G \circ h \mid h \in \mathcal H_{\widetilde{\ovk}}} \subset (\im(\mat S))^{\mathcal{X}},
  \end{equation}
  then $\mathcal{H}_{\widetilde{\ovk}, \mat S\mat G}$ is the RKHS associated to the operator-valued kernel $\ovk_{\mat S \mat G}$ defined in \Cref{eq:ovk_compose}. Moreover,
  \begin{equation}
    \norm{f}{\mathcal{H}_{\widetilde{\ovk}, \mat S\mat G}} = \min\set{\norm{h}{\mathcal H_{\widetilde{\ovk}}} \mid h \in \mathcal H_{\widetilde{\ovk}},  \mat S\mat G \circ h = f}.
  \end{equation}
\end{lemma}

\begin{proof}
  Define the operator $\op W: \mathcal{H}_{\widetilde{\ovk}} \to \mathcal{H}_{\widetilde{\ovk}, \mat S\mat G}$ such that $\op W(h) = \mat S \mat G \circ h$.
  Let $w = \mat S \mat G$, then the first part of the lemma can be proved by directly applying Proposition 6 of \citet{carmeli2010vector}. Moreover, by Proposition 1 of \citet{carmeli2010vector}, we know that $\op W$ is a partial isometry from $\mathcal H_{\widetilde{\ovk}}$ onto $\mathcal{H}_{\widetilde{\ovk}, \mat S\mat G}$, which proves the second part of the lemma.
\end{proof}

The following lemma establishes the existence and uniqueness of the minimizer $\hat{h}$ of Problem \eqref{eq:rvvkrr_fixed} and shows that it admits a finite kernel expansion.
\begin{lemma}[Representation Theorem] \label{lemma:rep_theo}
  Problem \eqref{eq:rvvkrr_fixed} admits a unique minimizer $\hat{h}$, which is of the form $\hat{h} =  \sum_{t=1}^{T} \widetilde{\ovk}_{x_t} c_t$  where $c_t \in \reals^d$ for all $t \in \intset{T}$.
\end{lemma}
\begin{proof}
  Let $\mathcal{U} = \set{h\mid h = \sum_{t=1}^T \widetilde{\ovk}_{x_t} c_t, (c_t)_{t=1}^T \in (\reals^d)^T}$ be a subspace of $\mathcal H_{\widetilde{\ovk}}$. Since $\mathcal{U}$ is finite-dimensional and hence closed, we can write $\mathcal H_{\widetilde{\ovk}} = \mathcal{U} \oplus \mathcal{U}^{\bot}$ where $\mathcal{U}^{\bot}$ denotes the orthogonal complement of $\mathcal{U}$. Any $h \in \mathcal H_{\widetilde{\ovk}}$ can be decomposed as $h = h_{\mathcal{U}} + h_{\mathcal{U}^{\bot}}$ with $h_{\mathcal{U}} \in \mathcal{U}$ and $h_{\mathcal{U}^{\bot}} \in \mathcal{U}^{\bot}$. Then by the reproducing property, we have for all $t \in \intset{T}$,
  \begin{equation}
    \inner{h_{\mathcal{U}^{\bot}}(x_t)}{y}{\reals^d} = \inner{\widetilde{\ovk}_{x_t}^{\adj} h_{\mathcal{U}^{\bot}}}{y}{\reals^d} = \inner{h_{\mathcal{U}^{\bot}}}{\widetilde{\ovk}_{x_t}y}{\mathcal H_{\widetilde{\ovk}}} = 0, \quad \forall y \in \reals^d.
  \end{equation}
  As a result, $h_{\mathcal{U}^{\bot}}(x_t) = 0$ for all $t \in \intset{T}$. For the first part of the objective function, we have
  \begin{equation}
    \norm{SGh(x_t) - y_t}{\reals^d}^2 = \norm{SG\left(h_{\mathcal{U}^{\bot}}(x_t) + h_{\mathcal{U}}(x_t)\right) - y_t}{\reals^d}^2 = \norm{SGh_{\mathcal{U}}(x_t) - y_t}{\reals^d}^2.
  \end{equation}
  For the regularization term, we have
  \begin{equation}
    \norm{h}{\mathcal H_{\widetilde{\ovk}}}^2 = \norm{h_{\mathcal{U}} + h_{\mathcal{U}^{\bot}}}{\mathcal H_{\widetilde{\ovk}}}^2 = \norm{h_{\mathcal{U}}}{\mathcal H_{\widetilde{\ovk}}}^2 + \norm{h_{\mathcal{U}^{\bot}}}{\mathcal H_{\widetilde{\ovk}}}^2 \geq \norm{h_{\mathcal{U}}}{\mathcal H_{\widetilde{\ovk}}}^2.
  \end{equation}
  Hence, if there exists a minimizer $\hat{h} \in \mathcal H_{\widetilde{\ovk}}$, we know that $\hat{h} \in \mathcal{U}$.
  
  Recall that our objective function is $L(h) = \frac{1}{T} \sum_{t=1}^T \norm{SG \widetilde{\ovk}_{x_t}^{\adj}h - y_t}{\reals^d}^2 + \lambda \norm{h}{\mathcal H_{\widetilde{\ovk}}}^2$. Since $\mathcal H_{\widetilde{\ovk}}$ is a reflexive Banach space, and $L:\mathcal H_{\widetilde{\ovk}} \to \reals$ is continuous, coercive, and strictly convex, we conclude that the minimum can be reached, and it admits a unique minimizer \citep[see, e.g., Theorem 2.19 of ][]{peypouquet_convex_2015}.
\end{proof}

\begin{proof}
  Denoting $L(f) \coloneqq \frac{1}{T} \sum_{t=1}^T \norm{f(x_t) - y_t}{\reals^d}^2 + \lambda \norm{f}{\mathcal{H}_{\widetilde{\ovk}, \mat S\mat G}}^2$, we want to prove that for all $f\in \mathcal{H}_{\widetilde{\ovk}, \mat S\mat G}$, $L(\hat{f}) \leq L(f)$. By \Cref{lemma:compose_ovk} and the fact that $\hat{f} = \mat S \mat G \circ \hat{h}$, we have 
  \begin{align*}
    L(\hat{f}) & = \frac{1}{T} \sum_{t=1}^T \norm{\mat S \mat G\hat{h}(x_t) - y_t}{\reals^d}^2 + \lambda \min\set{\norm{h}{\mathcal H_{\widetilde{\ovk}}}^2 \mid h \in \mathcal H_{\widetilde{\ovk}},  \mat S\mat G \circ h = \hat{f}} \\
    & \leq \frac{1}{T} \sum_{t=1}^T \norm{\mat S \mat G  \hat{h}(x_t) - y_t}{\reals^d}^2 + \lambda \norm{\hat{h}}{\mathcal H_{\widetilde{\ovk}}}^2.
  \end{align*}
  Since $\hat{h} = \argmin_{h\in \mathcal H_{\widetilde{\ovk}}} \frac{1}{T} \sum_{t=1}^T \norm{SG h(x_t) - y_t}{\reals^d}^2 + \lambda \norm{h}{\mathcal H_{\widetilde{\ovk}}}^2$, it follows that 
  \begin{align} \label{eq:ineq_two_minimizer}
    L(\hat{f}) \leq \frac{1}{T} \sum_{t=1}^T \norm{\mat S \mat G  h(x_t) - y_t}{\reals^d}^2 + \lambda \norm{h}{\mathcal H_{\widetilde{\ovk}}}^2, \quad \forall h \in \mathcal H_{\widetilde{\ovk}}.
  \end{align}
  Let $f \in \mathcal{H}_{\widetilde{\ovk}, \mat S\mat G}$, and let 
  \begin{equation*}
    \tilde{h} \coloneqq \argmin_{h \in \mathcal H_{\widetilde{\ovk}},  \mat S\mat G \circ h = f}\norm{h}{\mathcal H_{\widetilde{\ovk}}}^2.
  \end{equation*}
  Since the inequality described by \Cref{eq:ineq_two_minimizer} is valid for all $h \in \mathcal H_{\widetilde{\ovk}}$, we have then
  \begin{align*}
    L(\hat{f}) & \leq \frac{1}{T} \sum_{t=1}^T \norm{\mat S \mat G  \tilde{h} (x_t) - y_t}{\reals^d}^2 + \lambda \norm{\tilde{h} }{\mathcal H_{\widetilde{\ovk}}}^2 \\
    & = \frac{1}{T} \sum_{t=1}^T  \norm{f (x_t) - y_t}{\reals^d}^2  + \lambda \min\set{\norm{h}{\mathcal H_{\widetilde{\ovk}}}^2 \mid h \in \mathcal H_{\widetilde{\ovk}},  \mat S\mat G \circ h = f} \\
    & = L(f).
  \end{align*}
  The proof is thus concluded.
\end{proof}

\subsection{Link with WeaKL Estimators}
\label{app:weakl}
The general estimator proposed in \citet[Section 2]{doumeche_forecasting_2025} tries to solve the following optimization problem:
  \begin{equation}
    \min_{\theta \in \reals^{d_1}} \frac{1}{T} \sum_{t=1}^{T} \ell(\Phi(x_t) \theta, y_t) + \norm{\mat M \theta}{\reals^{d_2}}^2
  \end{equation}
  where, for each $i\in\intset{d}$, $\phi_i:\mathcal X\to\reals^{p_i}$ is a feature map with $p_i\in\intsplus$, and $d_1=\sum_{i=1}^d p_i$. Since in \citet[Proposition 2.1]{doumeche_forecasting_2025}, $\mat M$ is assumed to be injective, hence the inverse of $\mat M^{\trans} \mat M$ is well-defined. For any $x\in\mathcal X$, the feature matrix is defined as
  \begin{equation}  \Phi(x) \coloneqq
    \begin{bmatrix}
     \phi_1(x)^{\trans} & 0 & 0 \\
     0 & \ddots & 0 \\
     0 & 0 & \phi_{d}(x)^{\trans}
    \end{bmatrix} \in \reals^{d \times d_1}
  \end{equation}
  Here, $M \in \reals^{d_{2} \times d_1}$ is a regularization penalty matrix, and $\ell: \reals^d \times \reals^d \to \realsplus$ is a loss function defined as $\ell (y, y') = \norm{\Lambda (y - y')}{\reals^{d_3}}^2$ with $\Lambda \in \reals^{d_3 \times d}$. Since $\norm{\mat M \theta}{\reals^{d_2}}^2 = \inner{\mat M^{\trans} \mat M \theta}{\theta}{\reals^{d_1}} =  \norm{(\mat M^{\trans} \mat M)^{1/2} \theta}{\reals^{d_1}}^2$, denoting $w \coloneqq (\mat M^{\trans} \mat M)^{1/2} \theta \in \reals^{d_1}$, then the above optimization problem can be rewritten as
  \begin{equation}
    \min_{w \in \reals^{d_1}} \frac{1}{T} \sum_{t=1}^{T} \ell\left(\Phi(x_t) (\mat M^{\trans} \mat M)^{-1/2} w, y_t\right) + \norm{w}{\reals^{d_1}}^2.
  \end{equation}
  This problem is equivalent to the following vector-valued regression problem:
  \begin{equation} \label{eq:vvkrr_gen_loss}
    \min_{f \in \vvrkhs} \frac{1}{T} \sum_{t=1}^{T} \ell\left(f(x_t), y_t\right) + \norm{f}{\vvrkhs}^2.
  \end{equation}
  with the operator-valued kernel defined as, for all $(x,x') \in \mathcal{X}^2$,
  \begin{equation}
    \ovk(x,x') = \Phi(x) (\mat M^{\trans} \mat M)^{-1} \Phi(x')^{\trans} \in \reals^{d\times d}.
  \end{equation}

  To justify this equivalence, set $\Psi(x)=\Phi(x)(\mat M^{\trans}\mat M)^{-1/2}$, so that $\ovk(x,x')=\Psi(x)\Psi(x')^{\trans}$. By Proposition 1 of \citet{carmeli2010vector}, the associated RKHS consists of the functions $f_w=\Psi(\cdot)w$, with norm
  \begin{equation}
    \norm{f}{\vvrkhs}^2
    =\min\set{\norm{w}{\reals^{d_1}}^2\mid w\in\reals^{d_1},\ \Psi(x)w=f(x)\ \text{for all }x\in\mathcal X}.
  \end{equation}
  Since the empirical loss depends on $w$ only through $f_w$, minimizing the parameter objective first over all representations of each function yields exactly \Cref{eq:vvkrr_gen_loss}.
  
  The WeaKL-BU estimator proposed in \citet[Section 4.1]{doumeche_forecasting_2025} learns parameters only for the bottom-level series. Let $\theta_b\in\reals^{p_b}$ collect these parameters, let $\Phi_b(x)\in\reals^{d_b\times p_b}$ be their block-diagonal feature matrix, and let $\mat M_b\in\reals^{q_b\times p_b}$ be an injective penalty matrix. The predictor for all hierarchy levels is $f_{\theta_b}(x)=\mat S\Phi_b(x)\theta_b$, with penalty $\norm{\mat M_b\theta_b}{\reals^{q_b}}^2$. Applying the preceding construction to this parameterization shows that the resulting predictor is the minimizer of \Cref{eq:vvkrr_gen_loss} with the OVK
  \begin{equation}
    \ovk_{\mathrm{BU}}(x,x') = \mat S\Phi_b(x)(\mat M_b^{\trans}\mat M_b)^{-1}\Phi_b(x')^{\trans}\mat S^{\trans} \in \reals^{d\times d}.
  \end{equation}
  Similarly, the WeaKL-T estimator adds transfer learning constraints to the same bottom-level parameterization. We use the notation $J$, $D$, $\Pi_J$, and $P_J$ from \citet[Section 4.1]{doumeche_forecasting_2025}, to which we refer for their definitions. In particular, $P_J$ is the orthogonal projector onto the parameter-sharing subspace, so $(\mat I_{D|J|}-P_J)^{\trans}(\mat I_{D|J|}-P_J)=\mat I_{D|J|}-P_J$. The estimator corresponds to the minimizer with the OVK
  \begin{equation}
    \ovk_{\mathrm{T}}(x,x') = \mat S\Phi_b(x)\left(\mu\Pi_J^{\trans}(\mat I_{D|J|}-P_J)\Pi_J+\mat M_b^{\trans}\mat M_b\right)^{-1}\Phi_b(x')^{\trans}\mat S^{\trans} \in \reals^{d\times d},
  \end{equation}
  with $\mu \geq 0$ since $\mu\norms{(\mat I_{D|J|}-P_J)\Pi_J\theta_b}^2+\norms{\mat M_b\theta_b}^2=\norm{(\mu\Pi_J^{\trans}(\mat I_{D|J|}-P_J)\Pi_J+\mat M_b^{\trans}\mat M_b)^{1/2}\theta_b}{\reals^{p_b}}^2$.
  
  To express both kernels in the form of \Cref{eq:ovk_compose}, let $\mat G_{\mathrm{BU}}=[0_{d_b\times d_a}\;\;\mat I_{d_b}]$ and embed the bottom-level feature matrix into the full output space as $\overline{\Phi}(x)=\mat G_{\mathrm{BU}}^{\trans}\Phi_b(x)\in\reals^{d\times p_b}$. For either penalty matrix $\mat Q=\mat M_b^{\trans}\mat M_b$ or $\mat Q=\mu\Pi_J^{\trans}(\mat I_{D|J|}-P_J)\Pi_J+\mat M_b^{\trans}\mat M_b$, define the base OVK $\widetilde{\ovk}(x,x')=\overline{\Phi}(x)\mat Q^{-1}\overline{\Phi}(x')^{\trans}$. Since $\mat G_{\mathrm{BU}}\mat G_{\mathrm{BU}}^{\trans}=\mat I_{d_b}$, composing this base kernel with $\mat S\mat G_{\mathrm{BU}}$ yields the corresponding kernel above. Thus, both are special cases of \Cref{eq:ovk_compose}.

  However, it should be noted that the WeaKL estimators use a weighted loss function and allow complex-valued features, which can be easily extended into our RVVKRR framework.

\subsection{Computation Methods for Solving the Linear System in \Cref{eq:ls_vvkrr}}
\label{app:compu_vvkrr}
Now we describe the computational methods for solving the linear system in \Cref{eq:ls_vvkrr}. We will first consider the case of separable OVKs, which is more common in practice, and then discuss the non-separable case.

\subsubsection{Separable OVK Case}
In the case of a separable operator-valued kernel, i.e., $\ovk(x, x') = \kernel(x, x') \mat A$ where $\kernel: \mathcal X \times \mathcal X \to \reals$ is a scalar-valued kernel defined on the input space and the task kernel $\mat A \in \lineop{\im(\mat S)} \subset \reals^{d \times d}$ is self-adjoint and positive, \Cref{eq:ls_vvkrr} can be expressed in the following matrix form:
\begin{equation}
  T \lambda \mat C + \mat A \mat C \mat K = \mat{Y},
\label{eq:ls_vvkrr_sep}
\end{equation}
with  $\mat{Y} = \begin{bmatrix} y_1 & y_2 & \cdots  & y_T \end{bmatrix}\in \reals^{d\times T}$, $\mat{C} = \begin{bmatrix} c_1 & c_2 & \cdots  & c_T  \end{bmatrix}\in  \im(\mat S)^{T}$ and $\mat K = \left(\kernel(x_i, x_j)\right) \in \reals^{T\times T}$. As a result of \Cref{lemma:ovk_image_s}, \Cref{eq:ls_vvkrr_sep} becomes
\begin{equation} \label{eq:ls_vvkrr_sep_v2}
  T\lambda \mat E + \mat A' \mat E \mat K = \mat{U}_{\mat S}^{\trans} \mat Y.
\end{equation}
with $E = \mat{U}_{\mat S}^{\trans}\mat C \in \reals^{d_b \times T}$, which is a discrete-time Sylvester equation and can therefore be solved with the Hessenberg-Schur method \citep{golub_hessenberg-schur_1979, sima_algorithms_1996}, with a time complexity $\mathcal{O}(T^3+d_b^3)$ and a space complexity $\mathcal{O}(T^2+d_b^2)$\footnote{For example, we can use the implementation \textsc{SB04QD} provided by the SLICOT library (\url{https://www.slicot.org/}) and the corresponding python wrapper Slycot (\url{https://github.com/python-control/Slycot}). 
}. For a discrete-time Sylvester equation, the solution is unique if and only if $\sigma_i \mu_j \neq -1$ for all the eigenvalues $\sigma_i$ of $\mat A'$ and all the eigenvalues $\mu_j$ of $\mat K$ \citep[Section 6]{simincini_computa_2016}, which is satisfied here since $\mat A'$ and $\mat K$ are both positive semi-definite. Then, the prediction for a given input $x \in \mathcal{X}$ is given by:
  $\hat{f}(x) = \sum_{t=1}^T  \mat{U}_{\mat S} \mat A' \mat{U}_{\mat S}^{\trans} \kernel(x, x_t) c_t  = \mat{U}_{\mat S} \mat A' \mat E \vect k_x $
with $\vect{k}_x = \begin{bmatrix}
    k(x, x_1)  & \cdots & k(x, x_T) 
\end{bmatrix}^\trans  \in \reals^T$. 

\begin{remark}
    \Cref{eq:ls_vvkrr_sep_v2} provides another interpretation of our method. We first represent the training outputs as $y' = \mat{U}_{\mat S} y \in \reals^{d_b}$ in the new orthonormal basis $(u_1, \dots, u_{d_b})$ of $\im(S)$ .Using this transformed dataset, we then perform the VVKRR with a separable OVK, where $\mat A \in \mathbb{S}^{d_b}_+$.
\end{remark}
\subsubsection{Non-separable OVK Case}
In the case of a non-separable OVK, the Sylvester formulation in \Cref{eq:ls_vvkrr} is no longer applicable. A different computational route is available whenever the OVK admits an explicit finite-dimensional feature factorization. Suppose that, for some $m\in\intsplus$, there exists a feature operator
$\Gamma(x):\reals^m\longrightarrow\im(\mat S)$ such that
$\ovk(x,x')=\Gamma(x)\Gamma(x')^{\adj}$.
By Proposition 2.2 of \citet{carmeli_vector_2006}, the RKHS $\mathcal H_\ovk$ is then the set of all functions $f_\theta(x)=\Gamma(x)\theta$ with $\theta\in\reals^m$, and the RKHS norm is defined as $\norm{f_\theta}{\mathcal H_\ovk}
  =\min\set{\norm{\theta'}{\reals^m}\mid f_{\theta'}=f_\theta}$.
It follows that the RVVKRR defined in \Cref{eq:vvkrr_optim} can be obtained by solving the following finite-dimensional optimization problem:
\begin{equation}
  \widehat\theta
  =\argmin_{\theta\in\reals^m}
  \frac{1}{T}\sum_{t=1}^T
  \norm{\Gamma(x_t)\theta-y_t}{\reals^d}^2
  +\lambda\norm{\theta}{\reals^m}^2.
  \label{eq:nonsep_primal_objective}
\end{equation}
The first-order optimality condition gives the symmetric positive-definite linear system
\begin{equation}
  \left(
    \sum_{t=1}^T\Gamma(x_t)^{\adj}\Gamma(x_t)
    +T\lambda\mat I_m
  \right)\widehat\theta
  =
  \sum_{t=1}^T\Gamma(x_t)^{\adj}y_t.
  \label{eq:nonsep_primal_normal_equations}
\end{equation}
Although this system could be solved by explicitly forming and factorizing its $m\times m$ coefficient matrix, doing so can be prohibitive. We instead apply the preconditioned conjugate gradient method \citep{hestenes_methods_1952,cutajar_preconditioning_2016} in matrix-free form. Each iteration only requires evaluations of $\Gamma(x_t)$ and its adjoint; the matrix $\sum_{t=1}^T\Gamma(x_t)^{\adj}\Gamma(x_t)+T\lambda\mat I_m$ need not be materialized. In the OVK literature, this strategy has previously been used to solve large-scale discrete-time Sylvester equations induced by the separable matrix-valued kernel regression \citep{sindhwani_scalable_2013}. For example, consider the coherent OVK in \Cref{sec:ovk_compose} with the base kernel of \Cref{example:nonsep_base_kernel}. For a generic input $x=(x_1,\ldots,x_d)$, write $\phi_r=\phi(x_r)\in\reals^{p_\phi}$, $\ell_r=(\mat L_\rho)_{r,:}^{\trans}\in\reals^{q_\rho}$, $\theta=\mathrm{vec}_{\mathrm r}(\Theta)=\mathrm{vec}(\Theta^{\trans})$ with $\Theta\in\reals^{q_\rho\times p_\phi}$, and $\mat R=\mat S\mat G$. Since $\Gamma(x)=\mat R\Gamma_\rho(x)$, its forward evaluation is
\begin{equation}
  \Gamma(x)\theta
  =\mat R
  \begin{bmatrix}
    \lambda_1^{-1/2}\ell_1^{\trans}\Theta\phi_1\\
    \vdots\\
    \lambda_d^{-1/2}\ell_d^{\trans}\Theta\phi_d
  \end{bmatrix}
  \in\reals^d.
  \label{eq:nonsep_matrix_free_forward}
\end{equation}
and $\forall v\in\reals^d$, its adjoint evaluation is
\begin{equation}
  \Gamma(x)^{\adj}v
  =\mathrm{vec}_{\mathrm r}\!\left(
    \sum_{r=1}^d
    \frac{(\mat R^{\trans}v)_r}{\sqrt{\lambda_r}}
    \ell_r\phi_r^{\trans}
  \right)
  \in\reals^{q_\rho p_\phi}.
  \label{eq:nonsep_matrix_free_adjoint}
\end{equation}

We describe now the implementation in reduced coherent coordinates. Let
$\mat U_{\mat S}\in\reals^{d\times d_b}$ have orthonormal columns spanning
$\im(\mat S)$, let $\mat R=\mat S\mat G$, we have
\begin{equation}
  \norm{\mat R\Gamma_\rho(x_t)\Theta-y_t}{\reals^d}^2
  =\norm{\mat U_{\mat S}^{\trans}\mat R\Gamma_\rho(x_t)\Theta-\mat U_{\mat S}^{\trans}y_t}{\reals^{d_b}}^2
  +\norm{(\mat I_d-\mat U_{\mat S}\mat U_{\mat S}^{\trans})y_t}{\reals^d}^2.
  \label{eq:nonsep_reduced_loss_decomposition}
\end{equation}
Since the second term on the right-hand side is independent of $\Theta$, we can equivalently minimize the reduced loss described by the first term.

For $x=(x_1,\ldots,x_d)$, write $\phi_r=\phi(x_r)$ and $\ell_r=(\mat L_\rho)_{r,:}^{\trans}$. Now we define the reduced forward and adjoint operators $ \overline\Gamma_{\mat R,\rho}(x)$ and $\overline\Gamma_{\mat R,\rho}^{\adj}(x)$ by
\begin{align}
  \overline\Gamma_{\mat R,\rho}(x)\Theta
  &= \mat U_{\mat S}^{\trans}\mat R\Gamma_\rho(x) \Theta = \mat U_{\mat S}^{\trans}\mat R
  \begin{bmatrix}
    \lambda_1^{-1/2}\ell_1^{\trans}\Theta\phi_1\\
    \vdots\\
    \lambda_d^{-1/2}\ell_d^{\trans}\Theta\phi_d
  \end{bmatrix} \in \reals^{d_b}, \qquad \forall  \Theta\in\reals^{q_\rho\times p_\phi},
  \label{eq:nonsep_reduced_forward}\\
  \overline\Gamma_{\mat R,\rho}(x)^{\adj}c
  &= \Gamma_\rho(x)^{\adj} \mat R^{\trans}\mat U_{\mat S} c= \sum_{r=1}^d
    \frac{(\mat R^{\trans}\mat U_{\mat S}c)_r}
         {\sqrt{\lambda_r}}
    \ell_r\phi_r^{\trans} \in \reals^{q_\rho\times p_\phi},
  \qquad \forall c\in\reals^{d_b}.
  \label{eq:nonsep_reduced_adjoint}
\end{align}
With this notation, the linear system to solve can be rewritten as
\begin{equation}
  \left(\sum_{t=1}^T\overline\Gamma_{\mat R,\rho}(x_t)^{\adj}\overline\Gamma_{\mat R,\rho}(x_t)+T\lambda\idop_{\reals^{q_\rho\times p_\phi}}\right)\Theta
  = \sum_{t=1}^T\overline\Gamma_{\mat R,\rho}(x_t)^{\adj}\mat U_{\mat S}^{\trans}y_t.
  \label{eq:cg_reduced_linear_system}
\end{equation}

Let $p=p_\phi$, $q=q_\rho$. In the real implementation, it is more practical to define the batched forward operator $\mathcal F: \reals^{q\times p} \to \reals^{d_b\times T}$ and its adjoint $\mathcal F^{\adj}: \reals^{d_b\times T} \to \reals^{q\times p}$ as follows:
\begin{equation}
  \mathcal F(\Theta) =
  \begin{bmatrix}
    \overline\Gamma_{\mat R,\rho}(x_1)\Theta &
    \dots &
    \overline\Gamma_{\mat R,\rho}(x_T)\Theta
  \end{bmatrix}
  \in\reals^{d_b\times T},
  \label{eq:nonsep_batched_forward}
\end{equation}
and
\begin{equation}
  \mathcal F^{\adj}(\mat C) = \sum_{t=1}^T \overline\Gamma_{\mat R,\rho}(x_t)^{\adj}c_t,
  \quad \forall \mat C = \begin{bmatrix}c_1&\cdots&c_T\end{bmatrix} \in \reals^{d_b\times T}.
  \label{eq:nonsep_batched_adjoint}
\end{equation}
One can check that $\mathcal F^{\adj}$ is indeed the adjoint of $\mathcal F$ with respect to the Frobenius inner product, i.e., for all $\Theta \in \reals^{q\times p}$ and $\mat C \in \reals^{d_b\times T}$, we have
$
  \inner{\mathcal F(\Theta)}{\mat C}{\reals^{d_b\times T}} = \inner{\Theta}{\mathcal F^{\adj}(\mat C)}{\reals^{q\times p}}$.
As a result, \Cref{eq:cg_reduced_linear_system} can be rewritten as
\begin{equation}
  \left(\mathcal F^{\adj}\mathcal F+T\lambda\idop_{\reals^{q\times p}}\right)\Theta
  = \mathcal F^{\adj}(\overline{\mat Y}),
  \label{eq:cg_reduced_linear_system_batched}
\end{equation}
where $\mat Y=\begin{bmatrix}y_1&\cdots&y_T\end{bmatrix}\in\reals^{d\times T}$ and $\overline{\mat Y}=\mat U_{\mat S}^{\trans}\mat Y\in\reals^{d_b\times T}$.

We propose solving \Cref{eq:cg_reduced_linear_system_batched}
using the matrix-free preconditioned conjugate gradient method
\citep{hestenes_methods_1952,cutajar_preconditioning_2016},
with the preconditioner and custom initialization\footnote{We use the implementation provided by SciPy library (\url{https://docs.scipy.org/doc/scipy/reference/generated/scipy.sparse.linalg.cg.html}). }, where only the forward operator $\mathcal F$ and its adjoint $\mathcal F^{\adj}$ are needed. The following paragraphs describe how to compute $\mathcal F$ and $\mathcal F^{\adj}$ efficiently using tensor contractions.

\paragraph{Tensor Computation of $\mathcal{F}$ and $\mathcal{F}^{\adj}$}
Let $\mathsf X\in\reals^{T\times d\times p}$ be the feature tensor defined by
$\mathsf X_{t,r,:}=\phi(x_{t,r})^{\trans}$.
Define the intermediate matrix $\mat A_\Theta\in\reals^{d\times T}$ by
\begin{equation}
  \mat A_\Theta = \begin{bmatrix}
    \Gamma_\rho(x_1)\Theta & \cdots & \Gamma_\rho(x_T)\Theta
  \end{bmatrix} \in \reals^{d\times T}.
  \label{eq:cg_tensor_forward_intermediate_matrix}
\end{equation}
Then
\begin{equation}
  \mathcal F(\Theta)
  =\mat U_{\mat S}^{\trans}\mat R\mat A_\Theta
  \in\reals^{d_b\times T}.
  \label{eq:cg_tensor_forward_matrix}
\end{equation}
The entries of $\mat A_\Theta$ can be computed by the tensor contraction
\begin{equation}
  (\mat A_\Theta)_{rt}
  = \lambda_r^{-1/2}\ell_r^{\trans}\Theta\phi(x_{t,r})
  = \frac{1}{\sqrt{\lambda_r}}
   \sum_{j=1}^p\mathsf X_{t,r,j}(\mat L_\rho \Theta)_{rj}.
  \label{eq:cg_tensor_forward_intermediate}
\end{equation}
For $\mat C\in\reals^{d_b\times T}$, define the intermediate matrix $\mat G_{\mat C}\in\reals^{d\times p}$ by
\begin{equation}
  (\mat G_{\mat C})_{rj} = 
    \sum_{t=1}^T \frac{(\mat R^{\trans}\mat U_{\mat S}\mat C)_{rt}} {\sqrt{\lambda_r}}
    \phi(x_{t, r})_j
    = \frac{1}{\sqrt{\lambda_r}}
   \sum_{t=1}^T\mathsf X_{t,r,j}(\mat R^{\trans}\mat U_{\mat S}\mat C)_{rt}.
  \label{eq:cg_tensor_adjoint_intermediate_matrix}
\end{equation}
Then $\mathcal F^{\adj}(\mat C)=\mat L_\rho^{\trans}\mat G_{\mat C}$.
Moreover, to apply conjugate gradients, we identify $\Theta \in \reals^{q\times p}$ with
$\theta \in \reals^{qp}$ using row-wise vectorization.
\paragraph{Preconditioner} 
We denote by $\mathrm{vec}_{\mathrm r}$ the row-wise vectorization of a matrix, i.e., for $\Theta\in\reals^{q\times p}$, we have
\begin{equation}
  \vartheta_k=\Theta_{k,:}^{\trans}\in\reals^p,
  \qquad
  \theta=\mathrm{vec}_{\mathrm r}(\Theta)
  =\begin{bmatrix}\vartheta_1^{\trans}&\cdots&\vartheta_q^{\trans}\end{bmatrix}^{\trans}.
  \label{eq:cg_theta_blocks}
\end{equation}

Let $\mathrm{mat}_{\mathrm r}$ denote the inverse of the row-wise
vectorization introduced above. We denote by $\mat H\in\reals^{qp\times qp}$
the matrix representation of the operator
$\mathcal F^{\adj}\mathcal F+T\lambda\idop_{\reals^{q\times p}}$ and set
\begin{align}
  \mat H \theta
  &=\mathrm{vec}_{\mathrm r}\!\left(
    \mathcal F^{\adj}\mathcal F(\mathrm{mat}_{\mathrm r}(\theta))
    +T\lambda\,\mathrm{mat}_{\mathrm r}(\theta)\right),
  \label{eq:cg_matrix_free_product}\\
  \vect b
  &=\mathrm{vec}_{\mathrm r}\!\left(
    \mathcal F^{\adj}(\overline{\mat Y})\right).
  \label{eq:cg_matrix_free_rhs}
\end{align}
Thus, \Cref{eq:cg_reduced_linear_system_batched} is equivalent to
$\mat H\theta=\vect b$. Let $\mat H_{k\ell}\in\reals^{p\times p}$ denote block $(k,\ell)$ of
$\mat H$, then it can be shown that 
\begin{equation}
  \mat H_{k\ell}
  =\sum_{r=1}^d\sum_{s=1}^d
    \mat\Omega_{rs}\mat J_{rk}\mat J_{s\ell}
    \mat\Phi_r^{\trans}\mat\Phi_s
    +T\lambda\,\mathbf 1_{\{k=\ell\}}\mat I_p.
  \label{eq:cg_exact_normal_blocks}
\end{equation}
where
\begin{equation}
  \mat D_\lambda=\diag(\lambda_1,\ldots,\lambda_d),
  \qquad
  \mat J=\mat D_\lambda^{-1/2}\mat L_\rho,
  \qquad
  \mat\Omega
  =\mat R^{\trans}\mat U_{\mat S}\mat U_{\mat S}^{\trans}\mat R
  =\mat R^{\trans}\mat R,
  \label{eq:cg_exact_block_data}
\end{equation}
and $\mat\Phi_r\in\reals^{T\times p}$ the feature matrix of series $r$, whose
$t$-th row is $\phi(x_{t,r})^{\trans}$.

To increase the convergence speed of conjugate gradients, we use a preconditioner $\mat M_\rho\approx\mat H^{-1}$, which maps the residual $r_j$ to $z_j=\mat M_\rho r_j$. Following the general principle of approximate-inverse preconditioning for kernel systems \citep{cutajar_preconditioning_2016, rudi_falkon_2017}, we use one of two approximations depending on whether the kernel contains
task coupling. When $\rho>0$, we use the following Kronecker approximation of $\mat H$: 
\begin{equation}
  \widetilde{\mat H}_{\mathrm{task}}
  :=\mat A_{\mathrm{task}}\otimes\mat I_p \approx \mat H,
  \qquad
  \mat A_{\mathrm{task}}\in\reals^{q\times q}.
  \label{eq:cg_task_kronecker_approximation}
\end{equation}
where 
\begin{equation}
  (\mat A_{\mathrm{task}})_{k\ell} = \frac{1}{p} \trace{\mat H_{k\ell}} =\sum_{r=1}^d\sum_{s=1}^d
    \mat\Omega_{rs}\mat J_{rk}\mat J_{s\ell}
    \frac{1}{p}\trace{\mat\Phi_r^{\trans}\mat\Phi_s}
    +T\lambda\,\mathbf 1_{\{k=\ell\}}.
  \label{eq:cg_task_kronecker_approximation_blocks}
\end{equation}
It follows that
\begin{equation}
  \mat A_{\mathrm{task}}
  =\mat J^{\trans} (\mat\Omega\odot\mat C_X)\mat J+T\lambda\mat I_q.
\end{equation}
where $(\mat C_X)_{rs}
  =\frac{1}{p}\trace{\mat\Phi_r^{\trans}\mat\Phi_s}$.
  Thus the task-only preconditioner is
\begin{equation}
  \mat M_{\mathrm{task}}
  = \widetilde{\mat H}_{\mathrm{task}}^{-1} = \mat A_{\mathrm{task}}^{-1}\otimes\mat I_p.
  \label{eq:cg_task_only_preconditioner}
\end{equation}
It retains all interactions between the $q$ task factors and treats the feature directions as isotropic. When $\rho=0$, there is no task-factor mixing: $q=d$, $\mat L_0=\mat I_d$,
and $\mat J=\mat D_\lambda^{-1/2}$. We instead use a feature-only
block-Jacobi approximation that discards the off-diagonal blocks of $\mat H$, i.e.,
\begin{equation}
  \widetilde{\mat H}_{\mathrm{feature}}
  :=\operatorname{blockdiag}(\mat H_{11},\ldots,\mat H_{qq}) \approx \mat H.
  \label{eq:cg_feature_block_jacobi_approximation}
\end{equation}
In this case, \Cref{eq:cg_exact_normal_blocks} shows that its diagonal blocks
are exactly
\begin{equation}
  \mat A_k:=\mat H_{kk}
  =\frac{\mat\Omega_{kk}}{\lambda_k}
    \mat\Phi_k^{\trans}\mat\Phi_k+T\lambda\mat I_p.
  \label{eq:cg_feature_preconditioner_rho_zero}
\end{equation}
The feature-only preconditioner is therefore
\begin{equation}
  \mat M_{\mathrm{feature}}
  =\operatorname{blockdiag}
    (\mat A_1^{-1},\ldots,\mat A_q^{-1}).
  \label{eq:cg_implemented_preconditioner}
\end{equation}
Indeed, $\mat\Omega\succeq0$ and $\mat C_X\succeq0$, so the Schur product theorem and $T\lambda>0$ imply $\mat A_{\mathrm{task}}\succ0$ and $\mat A_k\succ0$ for every $k$; consequently, both $\mat M_{\mathrm{task}}$ and $\mat M_{\mathrm{feature}}$ are symmetric positive-definite and are valid PCG preconditioners.

\paragraph{Initialization} Instead of starting the conjugate gradient method from the zero vector, we propose to use a warm start based on the independently fitted KRR coefficients, which are denoted by
\begin{equation}
  \beta_r
  =\left(\mat\Phi_r^{\trans}\mat\Phi_r
    +T\lambda\,\lambda_r\mat I_p\right)^{-1}
    \mat\Phi_r^{\trans}\vect y_r.
  \label{eq:cg_independent_warm_start}
\end{equation}
where $\vect y_r=(y_{1,r},\ldots,y_{T,r})^{\trans}\in\reals^T$ and $\mat\Phi_r\in\reals^{T\times p}$ has rows $\phi(x_{t,r})^{\trans}$ as we defined before. Specifically, we begin from $\theta_0 = \mathrm{vec}_{\mathrm r}(\Theta_0)$ where $\Theta_0$ is constructed as follows:
\begin{equation}
  \Theta_0 \coloneqq \mat L_\rho^{\dagger} \begin{bmatrix}
    \sqrt{\lambda_1}\beta_1^{\trans}\\
    \vdots\\
    \sqrt{\lambda_d}\beta_d^{\trans}
  \end{bmatrix},
  \label{eq:cg_warm_start_mapping}
\end{equation}
where $\mat L_\rho^{\dagger}$ is the Moore--Penrose inverse. As a result, for $\rho=0$, $\mat L_0=\mat I_d$, it can be shown that $\Theta_0$ reproduces the $d$ independent-KRR predictors before the
coherent transformation $\mat U_{\mat S}^{\trans}\mat R$ is applied since
\begin{equation}
  (\Gamma_\rho(x) \Theta_0)_r = \lambda_r^{-1/2}(\Theta_0)_{r,:}\phi(x_r)
  =\beta_r^{\trans}\phi(x_r).
  \label{eq:cg_warm_start_scaling}
\end{equation}
For
$\rho>0$, $\Theta_0$ minimizes $\norm{\mat L_\rho\Theta-\mat W}{\reals^{d\times p}}^2$ where $\mat W=\begin{bmatrix}\sqrt{\lambda_1}\beta_1&\cdots&\sqrt{\lambda_d}\beta_d\end{bmatrix}^{\trans}$.

Once PCG returns $\widehat\theta$, we reshape it as $\widehat\Theta=\mathrm{mat}_{\mathrm r}(\widehat\theta)\in\reals^{q_\rho\times p_\phi}$ and predict, for any $x\in\mathcal X$,
\begin{equation}
  \widehat f(x)
  =\mat R\Gamma_\rho(x)\widehat\Theta
  =\mat U_{\mat S}\overline\Gamma_{\mat R,\rho}(x)\widehat\Theta
  \in\im(\mat S),
  \label{eq:cg_nonsep_prediction}
\end{equation}
so coherence holds by construction.

\section{Excess-Risk Bounds and Proofs}
\label{sec:excess_risk_proofs}
\paragraph{Notation for linear operators.}

Let $\mathcal{H}$ and $\mathcal{H}'$ be two Hilbert spaces. Let $(e_k)_{k\in I}$ be any orthonormal basis of $\mathcal{H}$. 

Let $\mathcal{B}(\mathcal{H}, \mathcal{H}')$ be the space of bounded linear operators from $\mathcal{H}$ to $\mathcal{H}'$. It's a normed space endowed with the operator norm $\norm{\cdot}{\mathrm{op}}$ defined by
$\norm{\op A}{\mathrm{op}} \coloneqq \sup_{h\in \mathcal{H}, \norm{h}{\mathcal{H}} \leq 1} \norm{\op A h}{\mathcal{H}'}$. Moreover, for any bounded operators $\op A, \op B$, when the output space of $\op B$ is the input space of $\op A$, we have
\begin{equation} \label{eq:ineq_norm_op}
  \norm{\op A \op B}{\mathrm{op}} \leq \norm{\op A}{\mathrm{op}} \norm{\op B}{\mathrm{op}}.
\end{equation} 
We define also $|\op A| \coloneqq (\op A^{\adj} \op A)^{\frac{1}{2}}$. 

Let $\mathcal{B}_0(\mathcal{H}, \mathcal{H}')$ be the space of compact bounded linear operators from $\mathcal{H}$ to $\mathcal{H}'$.

Let $\mathcal{B}_1(\mathcal{H}, \mathcal{H}')$ be the space of Schatten 1-class or trace class operators from $\mathcal{H}$ to $\mathcal{H}'$, i.e.,
\begin{equation}
  \mathcal{B}_1(\mathcal{H}, \mathcal{H}') = \set{\op A \in \mathcal{B}_0(\mathcal{H}, \mathcal{H}') \mid \sum_{k\in I} \inner{|A|e_k}{e_k}{\mathcal{H}} < \infty }.
\end{equation}
It is a Banach space endowed with the so-called trace class norm $\norm{\cdot}{1} $ defined by
$
\norm{\op A}{1} \coloneqq  \sum_{k\in I} \inner{|A|e_k}{e_k}{\mathcal{H}} = \sum_{k\in I} \sigma_{k} (\op A)
$,
where $\left(\sigma_{k} (A)\right)_{k \in I}$ are the singular values of $\op A$. For a trace class operator $\op A\in \mathcal{B}_1(\mathcal{H})$, its trace is well-defined:
$
  \trace{A} \coloneqq \sum_{l\in I} \inner{\op A e_l}{e_l}{\mathcal{H}} \leq \norm{\op A}{1} < \infty
$.

Let $\mathcal{B}_2(\mathcal{H}, \mathcal{H}')$ be the space of Schatten 2-class or Hilbert-Schmidt operators from $\mathcal{H}$ to $\mathcal{H}'$, i.e.,
\begin{equation}
  \mathcal{B}_2(\mathcal{H}, \mathcal{H}') = \set{ \op A \in  \mathcal{B}_0(\mathcal{H}, \mathcal{H}') \mid \sum_{k\in I} \inner{|A|^2 e_k}{e_k}{\mathcal{H}} < \infty }
\end{equation}
Then we can define a so-called Hilbert-Schmidt norm on $\mathcal{B}_2(\mathcal{H}, \mathcal{H}')$ as follows:
$
  \norm{\op A}{\mathrm{HS}} \coloneqq  \left(\sum_{k\in I} \inner{|A|^2e_k}{e_k}{\mathcal{H}} \right)^{1/2} = \left(\sum_{k\in I} \norm{\op A e_k}{\mathcal{H}'}^2 \right)^{1/2}=  \left( \sum_{k\in I} \sigma_{k}^2 (\op A) \right)^{1/2}
$.
For any  $(\op A, \op B) \in \mathcal{B}_2(\mathcal{H}, \mathcal{H}')$, $\op A^{\adj} \op B$ is in $\mathcal{B}_1(\mathcal{H})$. As a result, $\trace{\op A^{\adj} \op B}$ is well-defined. We can therefore endow $\mathcal{B}_2(\mathcal{H}, \mathcal{H}')$ with an inner product defined by 
$
  \inner{\op A}{ \op B}{2} \coloneqq \trace{\op A^{\adj} \op B}
$,
and show that it is a Hilbert space. The corresponding Hilbert-Schmidt norm is denoted by $\norm{\cdot}{2}$ and we have $\norm{\op A}{2} = \sqrt{\trace{\op A^{\adj} \op A}} = \norm{\op A}{\mathrm{HS}}$.

One can show that a trace-class operator is automatically Hilbert-Schmidt, so we have
\begin{equation}
  \mathcal{B}_1(\mathcal{H}, \mathcal{H}') \subset \mathcal{B}_2(\mathcal{H}, \mathcal{H}')  \subset \mathcal{B}_0(\mathcal{H}, \mathcal{H}').
\end{equation}
Moreover, for all $\op A \in \mathcal{B}_1(\mathcal{H}, \mathcal{H}')$, one has
\begin{equation}
 \norm{\op A}{\mathrm{op}}  \leq \norm{\op A}{2} \leq \norm{\op A}{1}.
\end{equation}
And for all $\op A\in \mathcal{B}(\mathcal{H})$, $\op B \in \mathcal{B}_p(\mathcal{H})$ with $p \in \set{1, 2}$, we have
\begin{equation} \label{eq:ine_norm_pclass_op}
  \norm{\op A \op B}{p} \leq \norm{\op A}{\mathrm{op}}\norm{\op B}{p}.
\end{equation}

Before presenting the main results of this section, we recall that it can be shown that the estimator of the RVVKRR can also be written in the following form:
\begin{equation}
  \hat{f} = (\hat{\Sigma}_{\ovk} + \lambda \idop)^{-1}\frac{1}{T}\sum_{t=1}^{T}\ovk_{x_t}y_t
\end{equation}
with 
\begin{equation}
  \hat{\Sigma}_{\ovk} \coloneqq \frac{1}{T}\sum_{t=1}^T\ovk_{x_t}\ovk_{x_t}^{\adj}.
\end{equation}
We introduce $f_{\lambda}$ which is a solution of the following optimization problem:
\begin{equation} \label{eq:f_lambda}
  f_{\lambda} \in \argmin_{f\in \vvrkhs}\set{\norm{f-\op S \circ g^{\targ}}{\rho}^2 + \lambda \norm{f}{\vvrkhs}^2}.
\end{equation}
Since $\lambda > 0$, a unique minimizer $f_{\lambda}$ exists, and is given by (see Proposition 1 of \citet{caponnetto_optimal_2007} for example):
\begin{equation}
  f_{\lambda} = (\Sigma_{\ovk}+ \lambda \idop)^{-1} \int_{\mathcal{X}} \ovk_{x} \op S \circ g^{\targ}(x) d\rho(x)
\label{eq:estimator_ridge_expect}
\end{equation}
where
\begin{equation}
  \Sigma_{\ovk} \coloneqq \int_{\mathcal{X}} \ovk_{x}\ovk_{x}^{\adj} d \rho(x).
\end{equation}
To ensure that the Bochner integral in \Cref{eq:estimator_ridge_expect} is well-defined, we assume additionally that $\vvrkhs$ is separable, and $(x, x') \mapsto \inner{\ovk_x y}{\ovk_{x'} y'}{\vvrkhs}$ is measurable for all $y, y' \in \mathcal{Y}$ \citep[see, e.g.,][Hypothesis 1]{caponnetto_optimal_2007}.

\subsection{Auxiliary Results}

First, we show some useful lemmas that are the direct consequences of the assumptions we made in \Cref{sec:stat_setting}.
\begin{lemma} \label{lemma:norm_funcs_assum_val}
  Under \Cref{assum:val}, we have
  \begin{align}
    \norm{\op S \circ g^{\targ}}{\rho} & \leq \norm{\op S}{\mathrm{op}} \sqrt{C_1}  \\
    \norm{f_{\lambda}-\op S \circ g^{\targ}}{\rho} &\leq \norm{\op S}{\mathrm{op}} \sqrt{C_1} \\
    \norm{f_{\lambda}}{\vvrkhs} &\leq \norm{\op S}{\mathrm{op}} \sqrt{\frac{C_1}{\lambda}}.
  \end{align}
\end{lemma}
\begin{proof}
  By definition of $f_{\lambda}$ (\Cref{eq:f_lambda}), we have
  \begin{equation}
    \norm{f_{\lambda}-\op S \circ g^{\targ}}{\rho}^2 + \lambda \norm{f_{\lambda}}{\vvrkhs}^2 \leq \norm{\op S \circ g^{\targ}}{\rho}^2.
  \end{equation}
  The fact that 
  \begin{align}
    \norm{\op S \circ g^{\targ}}{\rho}^2 & = \int_{\mathcal{X}} \norm{\op S \circ g^{\targ}(x)}{\mathcal{Y}}^2 d\rho(x) \leq \norm{\op S}{\mathrm{op}}^2 C_1  
  \end{align}
  concludes the proof.
\end{proof}

\begin{lemma} \label{lemma:bound_ovk}
  Under \Cref{assum:ovk}, $\ovk_{x}^{\adj}\ovk_{x}$ is a trace-class operator, and we have 
  \begin{equation}
    \norm{\ovk_{x}^{\adj}\ovk_{x}}{\mathrm{op}} \leq \norm{\ovk_{x}^{\adj}\ovk_{x}}{2} \leq \norm{\ovk_{x}^{\adj}\ovk_{x}}{1}  \leq \kappa 
  \end{equation}
\end{lemma}
\begin{proof}
  Since $\ovk_{x}^{\adj}\ovk_{x}$ is a self-adjoint, positive operator, we have $|\ovk_{x}^{\adj}\ovk_{x}| = \ovk_{x}^{\adj}\ovk_{x}$. Then by definition of the trace norm and the trace, $\norm{\ovk_{x}^{\adj}\ovk_{x}}{1} = \trace{\ovk_{x}^{\adj}\ovk_{x}} \leq \kappa $. The other inequalities are due to the properties of the trace-class operator.
\end{proof}

\begin{lemma}\label{lemma:excess_risk_identity}
  Under \Cref{assum:noise}, it can be shown that
\begin{equation}\label{eq:excess_risk_l2}
  \risk(\hat{f}) - \risk(\op S \circ g^{\targ}) = \expes{\norm{\hat{f}(\rv X) - \op S \circ  g^{\targ}(\rv X)}{ \mathcal{Y}}^2} = \norm{\hat{f} - \op S \circ g^{\targ}}{\rho}^2.
\end{equation}
\end{lemma}

The source condition in \Cref{assum:source_condition} is standard in the kernel literature and controls the smoothness of the target function. It is convenient for bounding the approximation error through the following lemma, which is similar to Proposition 3 in \citet{caponnetto_optimal_2007}. The scalar-valued kernel version and its proof can be found in Theorem 4 of \citet{smale_shannon_2005}; we omit the proof here.
\begin{lemma} \label{lemma:approx_error}
  Under \Cref{assum:source_condition}, we have
  \begin{equation}
    \norm{f_{\lambda} - \op S \circ g^{\targ}}{\rho} \leq \lambda^{r} \norm{L_{\ovk}^{-r}(\op S \circ g^{\targ})}{\rho}.
  \end{equation}
\end{lemma}

Before giving the explicit excess-risk bound and its proof, we present several
useful auxiliary results. We first restate Bennett's inequality for
Hilbert-space-valued random variables. This is a restatement of Lemma 2 in
\citet{smale_learning_2007}, derived from Theorem 3.4 of
\citet{pinelis_optimum_1994}.
\begin{proposition}[Bennett Inequality in Hilbert Space]\label{prop:benette_hilbert}
  Let $\mathcal{H}$ be a separable Hilbert space and let $\rv Z$ be a random variable on a probability space $(\Omega, \mathcal{F}, \rho)$ with values in $\mathcal{H}$. Assume that $\norm{\rv Z}{\mathcal{H}} \leq M < \infty$ almost surely. Denote $\sigma^2 = \expes{\norm{\rv Z}{\mathcal{H}}^2}$. Let $\set{\omega_i}_{i=1}^n$ be independent draws from $\rho$. For any $0 < \delta < 1$, with a probability greater than $1-\delta$, 
  \begin{equation} 
    \norm{\frac{1}{n}\sum_{i=1}^{n}(\rv Z (\omega_i) - \expes{\rv Z})}{\mathcal{H}} \leq \frac{2M \log(2/\delta)}{n} + \sqrt{\frac{2\sigma^2 \log(2/\delta)}{n}} \leq 2 \log\left(\frac{2}{\delta}\right) \left( \frac{M}{n} + \sqrt{\frac{\sigma^2}{n}}\right)
  \end{equation}
  The second inequality is due to the fact that $\sqrt{2\log(2/\delta)} < 2\log(2/\delta)$.

\end{proposition}

Using \Cref{prop:benette_hilbert}, the following two lemmas give some concentration results that are useful to prove the main proposition.

\begin{lemma} \label{lemma:ci_1}
  Let $0 < \delta < 1$. Under \Cref{assum:ovk}, with a probability greater than $1-\delta$, we have
  \begin{equation}
    \norm{\hat{\Sigma}_{\ovk} - \Sigma_{\ovk}}{2} \leq 4 \log\left(\frac{2}{\delta}\right) \frac{\kappa}{\sqrt{T}}.
  \end{equation}
  Moreover, we have
  \begin{equation}
    \norm{\left(\Sigma_{\ovk} + \lambda \idop \right)^{-1}\left(\hat{\Sigma}_{\ovk} - \Sigma_{\ovk}\right)}{2} \leq \frac{1}{2}
  \end{equation}
  when $\lambda \geq 8 \log\left(\frac{2}{\delta}\right) \frac{\kappa}{\sqrt{T}}$.
\end{lemma}
\begin{remark}
  The above lemma can be seen as the generalization of Lemma 4 in \citet{smale_learning_2007} in the case of OVK, and we give here a more concise proof.
\end{remark}
\begin{proof}
  We consider the random variable $\rv Z \coloneqq \ovk_{\rv X} \ovk_{\rv X}^{\adj}$ valued in $\mathcal{B}_2(\vvrkhs)$. Then by \Cref{lemma:bound_ovk}, $\norm{\rv Z}{2} = \norm{\ovk_{\rv X} \ovk_{\rv X}^{\adj}}{2} = \norm{\ovk_{\rv X}^{\adj}\ovk_{\rv X} }{2}  \leq \kappa$ and $\expes{\norm{\rv Z}{2}^2} \leq \kappa^2$. Since $\mathcal{B}_2(\vvrkhs)$ is a Hilbert space, applying \Cref{prop:benette_hilbert} to $\rv Z$, we obtain:
  \begin{equation*}
    \norm{\hat{\Sigma}_{\ovk} - \Sigma_{\ovk}}{2} \leq 2 \log\left(\frac{2}{\delta}\right) \left( \frac{\kappa}{T} + \frac{\kappa}{\sqrt{T}}\right) \leq 4 \log\left(\frac{2}{\delta}\right) \frac{\kappa}{\sqrt{T}}.
  \end{equation*}
  For the second bound, by \Cref{eq:ine_norm_pclass_op}, 
  \begin{align*}
    \norm{\left(\Sigma_{\ovk} + \lambda\idop\right)^{-1}\left(\hat{\Sigma}_{\ovk} - \Sigma_{\ovk}\right)}{2} & \leq \norm{\left(\Sigma_{\ovk} + \lambda\idop\right)^{-1}}{\mathrm{op}} \norm{\hat{\Sigma}_{\ovk} - \Sigma_{\ovk}}{2} \\
    & \leq 4 \log\left(\frac{2}{\delta}\right) \frac{\kappa}{\lambda\sqrt{T}}
  \end{align*}
  since $\norm{\left(\Sigma_{\ovk} + \lambda\idop\right)^{-1}}{\mathrm{op}} \leq \frac{1}{\lambda}$. And the proof is concluded by setting $4 \log\left(\frac{2}{\delta}\right) \frac{\kappa}{\lambda\sqrt{T}} \leq \frac{1}{2}$.
\end{proof}

\begin{lemma} \label{lemma:ci_2}
  Let $0 < \delta < 1$ and let $\lambda > 0$. Under \Cref{assum:val,assum:ovk,assum:noise}, with a probability greater than $1-\delta$, we have
  \begin{align}
  & \norm{(\Sigma_{\ovk} + \lambda \idop)^{-\frac{1}{2}} \left( \frac{1}{T}\sum_{t=1}^{T}\ovk_{x_t} \op S (g^{\targ}(x_t) + \xi_t) - \hat{\Sigma}_{\ovk} f_{\lambda} - \left(\int_{\mathcal{X}} \ovk_{x} \op S \circ g^{\targ}(x) d\rho(x) - \Sigma_{\ovk} f_{\lambda} \right)\right)}{\vvrkhs} \nonumber \\
  \leq & 2 \log\left(\frac{2}{\delta}\right) \left( \frac{1}{T}\sqrt{\frac{\kappa}{\lambda}} \norm{\op S}{\mathrm{op}}  \left( \left( 1 + \sqrt{\frac{\kappa}{\lambda}}\right)  \sqrt{C_1} + \sqrt{C_2} \right) + \sqrt{\frac{2 \left( \mathcal{N}(\lambda) \expes{\norm{\op S \xi }{\mathcal{Y}}^2}+ \frac{\kappa }{\lambda} \norm{\op S}{\mathrm{op}}^2 C_1 \right) }{T}}\right) \nonumber
\end{align}
with $\mathcal{N}(\lambda) \coloneqq \trace{(\Sigma_{\ovk} + \lambda \idop)^{-1} \Sigma_{\ovk}}$. $\mathcal{N}(\lambda)$ is often referred to as the effective dimension in the literature. Here, we write $y_t = \op S(g^{\targ}(x_t) + \xi_t)$ where $(x_t, \xi_t)$ are i.i.d. samples of $(\rv X, \xi)$.
\end{lemma}

\begin{proof}
  To apply \Cref{prop:benette_hilbert}, we define $\rv Z \coloneqq  (\Sigma_{\ovk} + \lambda \idop)^{-\frac{1}{2}} \ovk_{\rv X} \left(\op S \xi + \op S\circ g^{\targ}(\rv X) - \ovk^{\adj}_{\rv X} f_{\lambda} \right) $ valued in $\vvrkhs$. First, we try to bound the $\vvrkhs$-norm of the random variable $\rv Z$:
  \begin{align*}
    \norm{\rv Z}{\vvrkhs} & = \norm{(\Sigma_{\ovk} + \lambda \idop)^{-\frac{1}{2}} \ovk_{\rv X} \left(\op S \xi + \op S\circ g^{\targ}(\rv X) - \ovk^{\adj}_{\rv X} f_{\lambda} \right)}{\vvrkhs} \\
  & \leq \norm{(\Sigma_{\ovk} + \lambda \idop)^{-\frac{1}{2}} }{\mathrm{op}}  \norm{\ovk_{\rv X}}{\mathrm{op}} \norm{\op S \xi + \op S\circ g^{\targ}(\rv X) - \ovk^{\adj}_{\rv X} f_{\lambda} }{\mathcal{Y}} \\
  & \leq \norm{(\Sigma_{\ovk} + \lambda \idop)^{-\frac{1}{2}} }{\mathrm{op}} \sqrt{\kappa} \norm{\op S \xi + \op S\circ g^{\targ}(\rv X) - \ovk^{\adj}_{\rv X} f_{\lambda}}{\mathcal{Y}} && \text{(since $\norm{\ovk_{\rv X}}{\mathrm{op}} \leq \norm{\ovk_{\rv X}}{\mathrm{2}}$)} \\
  & \leq  \sqrt{\frac{\kappa}{\lambda}} \norm{\op S \xi + \op S\circ g^{\targ}(\rv X) - \ovk^{\adj}_{\rv X} f_{\lambda}}{\mathcal{Y}}  \\ 
  & \leq  \sqrt{\frac{\kappa}{\lambda}} \left(  \norm{\op S \xi}{\mathcal{Y}} +  \norm{\op S\circ g^{\targ}(\rv X)}{\mathcal{Y}} + \norm{ \ovk^{\adj}_{\rv X} f_{\lambda}}{\mathcal{Y}} \right) \\
  & \leq \sqrt{\frac{\kappa}{\lambda}} \left(\norm{\op S}{\mathrm{op}}  \norm{\xi}{\mathcal{Y}_b} +  \norm{\op S}{\mathrm{op}} \norm{g^{\targ}(\rv X)}{\mathcal{Y}_b} + \norm{ \ovk^{\adj}_{\rv X} f_{\lambda}}{\mathcal{Y}} \right) \\
  & \leq \sqrt{\frac{\kappa}{\lambda}} \left(\norm{\op S}{\mathrm{op}} \sqrt{C_2} +  \norm{\op S}{\mathrm{op}} \sqrt{C_1} + \norm{ \ovk^{\adj}_{\rv X} f_{\lambda}}{\mathcal{Y}} \right) && \text{(by \Cref{assum:val}).}
  \end{align*}
  \Cref{lemma:norm_funcs_assum_val} and \Cref{assum:ovk} give $\norm{ \ovk^{\adj}_{\rv X} f_{\lambda}}{\mathcal{Y}} \leq \norm{ \ovk^{\adj}_{\rv X}}{\mathrm{op}} \norm{f_{\lambda}}{\vvrkhs} \leq \norm{\op S}{\mathrm{op}} \sqrt{\frac{\kappa C_1}{\lambda}}$. As a result, we have
  \begin{equation*}
    \norm{\rv Z}{\vvrkhs} \leq \sqrt{\frac{\kappa}{\lambda}}\norm{\op S}{\mathrm{op}} \left( \left( 1 + \sqrt{\frac{\kappa}{\lambda}}\right)  \sqrt{C_1} + \sqrt{C_2} \right).
  \end{equation*}
  Then we focus on obtaining an upper bound of $\expes{\norm{\rv Z}{\vvrkhs}^2}$. Denoting $\rv U \coloneqq \op S \xi + \op S\circ g^{\targ}(\rv X) - \ovk^{\adj}_{\rv X} f_{\lambda} $ valued in $\mathcal{Y}$, we can rewrite $\expes{\norm{\rv Z}{\vvrkhs}^2}$ as follows:
  \begin{equation*}
    \expes{\norm{\rv Z}{\vvrkhs}^2} =  \expes{\inner{\ovk_{\rv X}^{\adj} (\Sigma_{\ovk} + \lambda \idop)^{-1}\ovk_{\rv X}\rv U }{\rv U}{\mathcal{Y}}}.
  \end{equation*}
  It can be shown that $\ovk_{\rv X}^{\adj} (\Sigma_{\ovk} + \lambda \idop)^{-1}\ovk_{\rv X}$ is self-adjoint, which yields:
\begin{align*}
  \expes{\norm{\rv Z}{\vvrkhs}^2} & \leq \expes{\norm{\ovk_{\rv X}^{\adj} (\Sigma_{\ovk} + \lambda \idop)^{-1}\ovk_{\rv X}}{\mathrm{op}}   \norm{\rv U }{\mathcal{Y}}^2}.
\end{align*}
The fact that $\norm{\rv U }{\mathcal{Y}}^2 \leq 2 \norm{\op S \xi }{\mathcal{Y}}^2 + 2 \norm{\op S\circ g^{\targ}(\rv X) - \ovk^{\adj}_{\rv X} f_{\lambda}}{\mathcal{Y}}^2 $ results in
\begin{align*}
  \expes{\norm{\rv Z}{\vvrkhs}^2} & \leq 2 \left(\expes{\norm{\ovk_{\rv X}^{\adj} (\Sigma_{\ovk} + \lambda \idop)^{-1}\ovk_{\rv X}}{\mathrm{op}} \norm{\op S \xi }{\mathcal{Y}}^2} + \expes{\norm{\ovk_{\rv X}^{\adj} (\Sigma_{\ovk} + \lambda \idop)^{-1}\ovk_{\rv X}}{\mathrm{op}}   \norm{\op S\circ g^{\targ}(\rv X) - \ovk^{\adj}_{\rv X} f_{\lambda}}{\mathcal{Y}}^2} \right).
\end{align*}
It follows that
\begin{align*}
  & \expes{\norm{\ovk_{\rv X}^{\adj} (\Sigma_{\ovk} + \lambda \idop)^{-1}\ovk_{\rv X}}{\mathrm{op}} \norm{\op S \xi }{\mathcal{Y}}^2} \\
  = & \expes{\norm{\ovk_{\rv X}^{\adj} (\Sigma_{\ovk} + \lambda \idop)^{-1}\ovk_{\rv X}}{\mathrm{op}}} \expes{\norm{\op S \xi }{\mathcal{Y}}^2} && \text{(by \Cref{assum:noise})} \\
  \leq & \expes{\trace{\ovk_{\rv X}^{\adj} (\Sigma_{\ovk} + \lambda \idop)^{-1}\ovk_{\rv X}}}  \expes{\norm{\op S \xi }{\mathcal{Y}}^2}  && \text{(since $\ovk_{\rv X}^{\adj} (\Sigma_{\ovk} + \lambda \idop)^{-1}\ovk_{\rv X}$ is positive)} \\
  = & \expes{\trace{(\Sigma_{\ovk} + \lambda \idop)^{-1}\ovk_{\rv X}\ovk_{\rv X}^{\adj}  }}  \expes{\norm{\op S \xi }{\mathcal{Y}}^2} \\
  = & \mathcal{N}(\lambda) \expes{\norm{\op S \xi }{\mathcal{Y}}^2}  && \text{(by definition of $\mathcal{N}(\lambda)$)}
\end{align*}
and
\begin{align*}
  & \expes{\norm{\ovk_{\rv X}^{\adj} (\Sigma_{\ovk} + \lambda \idop)^{-1}\ovk_{\rv X}}{\mathrm{op}}   \norm{\op S\circ g^{\targ}(\rv X) - \ovk^{\adj}_{\rv X} f_{\lambda}}{\mathcal{Y}}^2} \\
  \leq & \expes{\norm{\ovk_{\rv X}^{\adj}}{\mathrm{op}}\norm{(\Sigma_{\ovk} + \lambda \idop)^{-1}}{\mathrm{op}}\norm{\ovk_{\rv X}}{\mathrm{op}}   \norm{\op S\circ g^{\targ}(\rv X) - \ovk^{\adj}_{\rv X} f_{\lambda}}{\mathcal{Y}}^2} && \text{(\Cref{eq:ineq_norm_op})} \\
  \leq & \expes{ \frac{\kappa}{\lambda} \norm{\op S\circ g^{\targ}(\rv X) - \ovk^{\adj}_{\rv X} f_{\lambda}}{\mathcal{Y}}^2} \\
  = & \frac{\kappa}{\lambda} \norm{\op S \circ g^{\targ} - f_{\lambda}}{\rho}^2  \\
  \leq & \frac{\kappa \norm{\op S}{\mathrm{op}}^2 C_1}{\lambda} && \text{(\Cref{lemma:norm_funcs_assum_val})} 
\end{align*}
As a result, we obtain
  \begin{equation*}
    \expes{\norm{\rv Z}{\vvrkhs}^2} \leq 2 \left( \mathcal{N}(\lambda) \expes{\norm{\op S \xi }{\mathcal{Y}}^2}+ \frac{\kappa \norm{\op S}{\mathrm{op}}^2 C_1}{\lambda}\right).
  \end{equation*}
We conclude by applying \Cref{prop:benette_hilbert} to the random variable $\rv Z$, using the following expression for $\expes{\rv Z}$:
\begin{align*}
    \expes{\rv Z}  &=  (\Sigma_{\ovk} + \lambda \idop)^{-\frac{1}{2}} \left(\expes{\ovk_{\rv X} \op S \xi }  + \expes{\ovk_{\rv X} \op S\circ g^{\targ}(\rv X) } - \expes{\ovk_{\rv X} \ovk^{\adj}_{\rv X} f_{\lambda} } \right)\\
    &=  (\Sigma_{\ovk} + \lambda \idop)^{-\frac{1}{2}} \left( \expes{\ovk_{\rv X} \op S\circ g^{\targ}(\rv X) } - \expes{\ovk_{\rv X} \ovk^{\adj}_{\rv X} f_{\lambda} } \right)  && \text{(\Cref{assum:noise})} \\ 
    & = (\Sigma_{\ovk} + \lambda \idop)^{-\frac{1}{2}} \left( \int_{\mathcal{X}} \ovk_{x} \op S \circ g^{\targ}(x) d\rho(x) - \Sigma_{\ovk} f_{\lambda}  \right)
  \end{align*}
\end{proof}

We also need the following lemma which links the norms $\norm{\cdot}{\rho}$ and $\norm{\cdot}{\vvrkhs}$.
\begin{lemma} \label{lemma:norm_l2_rkhs}
  For any $(f, g) \in \vvrkhs^2$, we have
  \begin{equation}
    \inner{f}{g}{\rho} = \inner{f}{\Sigma_{\ovk} g}{\vvrkhs}
  \end{equation}
  with $\Sigma_{\ovk} \coloneqq \int_{\mathcal{X}} \ovk_{x}\ovk_{x}^{\adj} d \rho(x) \in \lineop{\vvrkhs}$. As a result, $\norm{f}{\rho} = \norm{\Sigma_{\ovk}^{1/2} f}{\vvrkhs}$.
\end{lemma}
\begin{proof}
  By the definition of the $L^2(\mathcal{X}, \rho, \mathcal{Y})$-inner product, we have,  for any $(f, g) \in \vvrkhs^2$,  
  \begin{align*}
  \inner{f}{g}{\rho} & = \int_{\mathcal{X}} \inner{f(x)}{g(x)}{\mathcal{Y}} d\rho(x) \\
  & =  \int_{\mathcal{X}} \inner{\ovk_{x}^{\adj} f}{\ovk_{x}^{\adj} g}{\mathcal{Y}} d\rho(x) \\
  & = \int_{\mathcal{X}} \inner{f}{\ovk_{x} \ovk_{x}^{\adj} g}{\vvrkhs} d\rho(x) \\
  & = \inner{f}{ \left(\int_{\mathcal{X}} \ovk_{x} \ovk_{x}^{\adj} d\rho(x)\right) g}{\vvrkhs} 
  \end{align*}
  The proof is concluded with the definition of $\Sigma_{\ovk}$.
\end{proof}

\subsection{Explicit Excess-Risk Bound}

\begin{proposition}[Explicit excess-risk bound]
\label{prop:bound_excess_risk_explicit}
  Under \Cref{assum:val,assum:ovk,assum:noise,assum:source_condition}, let $0 < \delta < 1$, then with a probability greater than $1-\delta$,
  \begin{align*}
    \risk(\hat{f}) - \risk(\op S \circ g^{\targ})
    \leq{}& 128 \log^2\left(\frac{4}{\delta}\right)
    \Bigg[
      \frac{\kappa\norm{\op S}{\mathrm{op}}^2}{\lambda T^2}
      \left(\left(1+\sqrt{\frac{\kappa}{\lambda}}\right)^2C_1+C_2\right)
      +\frac{1}{T}
      \left(
        \mathcal{N}(\lambda)\expes{\norm{\op S\xi}{\mathcal{Y}}^2}
        +\frac{\kappa}{\lambda}\norm{\op S}{\mathrm{op}}^2C_1
      \right)
    \Bigg] \\ 
    & +2\lambda^{2r}
      \norm{L_{\ovk}^{-r}(\op S\circ g^{\targ})}{\rho}^2,
  \end{align*}
  provided that
  \begin{equation}
    \lambda \geq 8 \log\left(\frac{4}{\delta}\right) \frac{\kappa}{\sqrt{T}}.
  \end{equation}
\end{proposition}

\begin{proof}
First we decompose $\hat{f} - \op S \circ g^{\targ} $ by 
\begin{equation*}
   \hat{f} - \op S \circ g^{\targ} = \hat{f} - f_{\lambda} + f_{\lambda} - \op S \circ g^{\targ}.
\end{equation*}
From \Cref{eq:excess_risk_l2} and due to the inequality $\norm{f_1 + f_2}{\rho}^2 \leq 2(\norm{f_1}{\rho}^2 + \norm{f_2}{\rho}^2 )$ for any $(f_1, f_2) \in L^2(\mathcal{X}, \rho, \mathcal{Y})^2$, the excess risk can be upper bounded by
\begin{align*}
  \risk(\hat{f}) - \risk(\op S \circ g^{\targ}) \leq 2 \left(\norm{\hat{f} - f_{\lambda}}{\rho}^2 + \norm{f_{\lambda} - \op S \circ g^{\targ}}{\rho}^2 \right)
\end{align*}

The first step of the proof is dedicated to obtain the probabilistic bound on $\norm{\hat{f} - f_{\lambda}}{\rho}$. By the definition of $\hat{f}$, we have
\begin{align*}
   \hat{f} - f_{\lambda} & =  (\hat{\Sigma}_{\ovk} + \lambda \idop)^{-1}\frac{1}{T}\sum_{t=1}^{T}\ovk_{x_t} \op S (g^{\targ}(x_t) + \xi_t) - f_{\lambda} \\
   & = (\hat{\Sigma}_{\ovk} + \lambda \idop)^{-1} \left( \frac{1}{T}\sum_{t=1}^{T}\ovk_{x_t} \op S (g^{\targ}(x_t) + \xi_t) - (\hat{\Sigma}_{\ovk} + \lambda \idop) f_{\lambda}  \right) \\
 & = (\hat{\Sigma}_{\ovk} + \lambda \idop)^{-1} \left( \frac{1}{T}\sum_{t=1}^{T}\ovk_{x_t} \op S (g^{\targ}(x_t) + \xi_t) - (\hat{\Sigma}_{\ovk} - \Sigma_{\ovk}) f_{\lambda}  - (\Sigma_{\ovk} + \lambda \idop) f_{\lambda}  \right) 
\end{align*}
It follows that
\begin{align*}
  & \hat{f} - f_{\lambda} \\
  = & (\hat{\Sigma}_{\ovk} + \lambda \idop)^{-1} \left( \frac{1}{T}\sum_{t=1}^{T}\ovk_{x_t} \op S (g^{\targ}(x_t) + \xi_t) - (\hat{\Sigma}_{\ovk} - \Sigma_{\ovk}) f_{\lambda}  - \int_{\mathcal{X}} \ovk_{x} \op S \circ g^{\targ}(x) d\rho(x)   \right) \\
  = &  (\hat{\Sigma}_{\ovk} + \lambda \idop)^{-1} \left( \frac{1}{T}\sum_{t=1}^{T}\ovk_{x_t} \op S (g^{\targ}(x_t) + \xi_t) - \hat{\Sigma}_{\ovk} f_{\lambda} - \left(\int_{\mathcal{X}} \ovk_{x} \op S \circ g^{\targ}(x) d\rho(x) - \Sigma_{\ovk} f_{\lambda} \right)\right)
\end{align*}
since  $(\Sigma_{\ovk}+ \lambda \idop)f_{\lambda} = \int_{\mathcal{X}} \ovk_{x} \op S \circ g^{\targ}(x) d\rho(x) $ from \Cref{eq:estimator_ridge_expect}. Denoting $u = \frac{1}{T}\sum_{t=1}^{T}\ovk_{x_t} \op S (g^{\targ}(x_t) + \xi_t) - \hat{\Sigma}_{\ovk} f_{\lambda} - \left(\int_{\mathcal{X}} \ovk_{x} \op S \circ g^{\targ}(x) d\rho(x) - \Sigma_{\ovk} f_{\lambda} \right)$, $\norm{\hat{f} - f_{\lambda}}{\rho}$ can be bounded as follows
\begin{align*}
  \norm{\hat{f} - f_{\lambda}}{\rho} & = \norm{(\hat{\Sigma}_{\ovk} + \lambda \idop)^{-1} u}{\rho} \\
  & = \norm{\Sigma_{\ovk}^{1/2} (\hat{\Sigma}_{\ovk} + \lambda \idop)^{-1} u}{\vvrkhs} && \text{(\Cref{lemma:norm_l2_rkhs})} \\
  & \leq \norm{\Sigma_{\ovk}^{\frac{1}{2}} (\hat{\Sigma}_{\ovk} + \lambda \idop)^{-1} (\Sigma_{\ovk} + \lambda \idop)^{\frac{1}{2}} }{\mathrm{op}}\norm{(\Sigma_{\ovk} + \lambda \idop)^{-\frac{1}{2}} u}{\vvrkhs}.
\end{align*}

To bound $\norm{\Sigma_{\ovk}^{\frac{1}{2}} (\hat{\Sigma}_{\ovk} + \lambda \idop)^{-1} (\Sigma_{\ovk} + \lambda \idop)^{\frac{1}{2}} }{\mathrm{op}}$, Step 3.1 in the proof of Theorem 4 of \citet{caponnetto_optimal_2007} shows 
\begin{equation*}
  \norm{\Sigma_{\ovk}^{\frac{1}{2}} (\hat{\Sigma}_{\ovk} + \lambda \idop)^{-1} (\Sigma_{\ovk} + \lambda \idop)^{\frac{1}{2}} }{\mathrm{op}} \leq 2 
\end{equation*}
provided that $\norm{\left(\Sigma_{\ovk} + \lambda \idop \right)^{-1}\left(\hat{\Sigma}_{\ovk} - \Sigma_{\ovk}\right)}{2} \leq \frac{1}{2}$. \Cref{lemma:ci_1} implies that it holds, with a probability greater than $1-\delta/2$, when $\lambda \geq 8 \log\left(\frac{4}{\delta}\right) \frac{\kappa}{\sqrt{T}}$. To bound $\norm{(\Sigma_{\ovk} + \lambda \idop)^{-\frac{1}{2}} u}{\vvrkhs}$, we apply \Cref{lemma:ci_2}, which indicates, with a probability greater than $1-\delta/2$,
\begin{align*}
  & \norm{(\Sigma_{\ovk} + \lambda \idop)^{-\frac{1}{2}} u}{\vvrkhs} \\ \leq & 2 \log\left(\frac{4}{\delta}\right) \left( \frac{1}{T}\sqrt{\frac{\kappa}{\lambda}} \norm{\op S}{\mathrm{op}}  \left( \left( 1 + \sqrt{\frac{\kappa}{\lambda}}\right)  \sqrt{C_1} + \sqrt{C_2} \right) + \sqrt{\frac{2 \left( \mathcal{N}(\lambda) \expes{\norm{\op S \xi }{\mathcal{Y}}^2}+ \frac{\kappa }{\lambda} \norm{\op S}{\mathrm{op}}^2 C_1 \right) }{T}}\right).
\end{align*}
It follows that with a probability greater than $1 - \delta$, when $\lambda \geq 8 \log\left(\frac{4}{\delta}\right) \frac{\kappa}{\sqrt{T}}$,
\begin{align*}
  \norm{\hat{f} - f_{\lambda}}{\rho}^2 \leq 64 \log^2\left(\frac{4}{\delta}\right) \left ( \frac{\kappa\norm{\op S}{\mathrm{op}}^2}{\lambda T^2} \left(( 1 + \sqrt{\frac{\kappa}{\lambda}})^2 C_1 + C_2\right) + \frac{1}{T} \left(\mathcal{N}(\lambda) \expes{\norm{\op S \xi }{\mathcal{Y}}^2}+ \frac{\kappa }{\lambda} \norm{\op S}{\mathrm{op}}^2 C_1\right) \right)
\end{align*}
since $(\frac{1}{T}\sqrt{\frac{\kappa}{\lambda}} \norm{\op S}{\mathrm{op}}  \left( \left( 1 + \sqrt{\frac{\kappa}{\lambda}}\right)  \sqrt{C_1} + \sqrt{C_2} \right))^2 \leq \frac{2\kappa\norm{\op S}{\mathrm{op}}^2}{\lambda T^2} \left(\left( 1 + \sqrt{\frac{\kappa}{\lambda}}\right)^2 C_1 + C_2\right) $.

The second step to upper bound on the approximation error $\norm{f_{\lambda} - \op S \circ g^{\targ}}{\rho}$, which can be achieved by applying directly \Cref{lemma:approx_error}. 

The proof is concluded by combining the two bounds.

\end{proof}

\subsection{Derivation of the Compact Bound}

We now derive the compact bound stated in the main paper from
\Cref{prop:bound_excess_risk_explicit}.
\begin{proof}[Proof of \Cref{prop:bound_excess_risk}]
Let $\ell_\delta\coloneqq\log(4/\delta)$. The condition on $\lambda$ in
\Cref{prop:bound_excess_risk} implies
\begin{equation*}
  \frac{\kappa}{\lambda}
  \leq \frac{\sqrt{T}}{8\ell_\delta}
  \leq \frac{\sqrt{T}}{8},
\end{equation*}
where the second inequality follows from
$\ell_\delta>\log(4)>1$. Since $T\geq 1$, we therefore have
\begin{align*}
  \frac{1}{T}
  \left(1+\sqrt{\frac{\kappa}{\lambda}}\right)^2
  \leq
  \frac{1}{T}\left(1+\frac{T^{1/4}}{\sqrt{8}}\right)^2
  \leq
  \frac{1}{\sqrt{T}}\left(1+\frac{1}{\sqrt{8}}\right)^2
  \leq 2.
\end{align*}
Using the notation introduced in \Cref{prop:bound_excess_risk}, this yields
\begin{align*}
\frac{\kappa\norm{\op S}{\mathrm{op}}^2}{\lambda T^2}
  \left[
    \left(1+\sqrt{\frac{\kappa}{\lambda}}\right)^2 C_1+C_2
  \right]
  = \frac{\kappa}{\lambda T}
  \left[
    \frac{1}{T}
    \left(1+\sqrt{\frac{\kappa}{\lambda}}\right)^2 B_S^2
    +\frac{M_S^2}{T}
  \right] \leq
  \frac{\kappa}{\lambda T}(2B_S^2+M_S^2).
\end{align*}
Furthermore,
\begin{equation*}
  \frac{1}{T}
  \left[
    \mathcal{N}(\lambda)
    \expes{\norm{\op S\xi}{\mathcal{Y}}^2}
    +\frac{\kappa}{\lambda}\norm{\op S}{\mathrm{op}}^2C_1
  \right]
  =\frac{1}{T}
  \left[
    \mathcal{N}(\lambda)\sigma_S^2
    +\frac{\kappa}{\lambda}B_S^2
  \right].
\end{equation*}
It follows that the bound in \Cref{prop:bound_excess_risk_explicit} can be simplified to, with probability at least
$1-\delta$,
\begin{align*}
  \risk(\hat{f})-\risk(\op S\circ g^{\targ})
  \leq
  \frac{128}{T}\log^2\!\left(\frac{4}{\delta}\right)
  \left[
    \mathcal{N}(\lambda)\sigma_S^2
    +\frac{\kappa}{\lambda}(3B_S^2+M_S^2)
  \right] +2\lambda^{2r}R_r^2,
\end{align*}
which is precisely the claim of \Cref{prop:bound_excess_risk}.
\end{proof}

\section{Additional Details on Numerical Experiments}
\label{app:add_detail_exp}
All experiments were conducted on a Linux-based personal computer equipped with an Intel Core Ultra 7 165U processor (12 cores, 14 threads, up to 4.9 GHz). No GPU was used.

\paragraph{Evaluation Metrics.} 
To evaluate methods in the numerical experiments, we report the mean squared error (MSE) computed at each hierarchical level as well as across all levels. Specifically, given an observed time series $(y_t)_{t=1}^T$ with $y_t \in \reals^r$ where $r$ denotes either the number of series at a given hierarchical level or the total number of series across all levels, and corresponding predictions $(\hat{y}_t)_{t=1}^T$, the MSE is defined as $\frac{1}{T}\sum_{t=1}^T \norm{y_t - \hat{y}_t}{\reals^r}^2$. 

\subsection{Synthetic dataset}
\label{app:toy_hts}
\begin{figure}[htbp!]
    \centering
    \includegraphics[width=\linewidth]{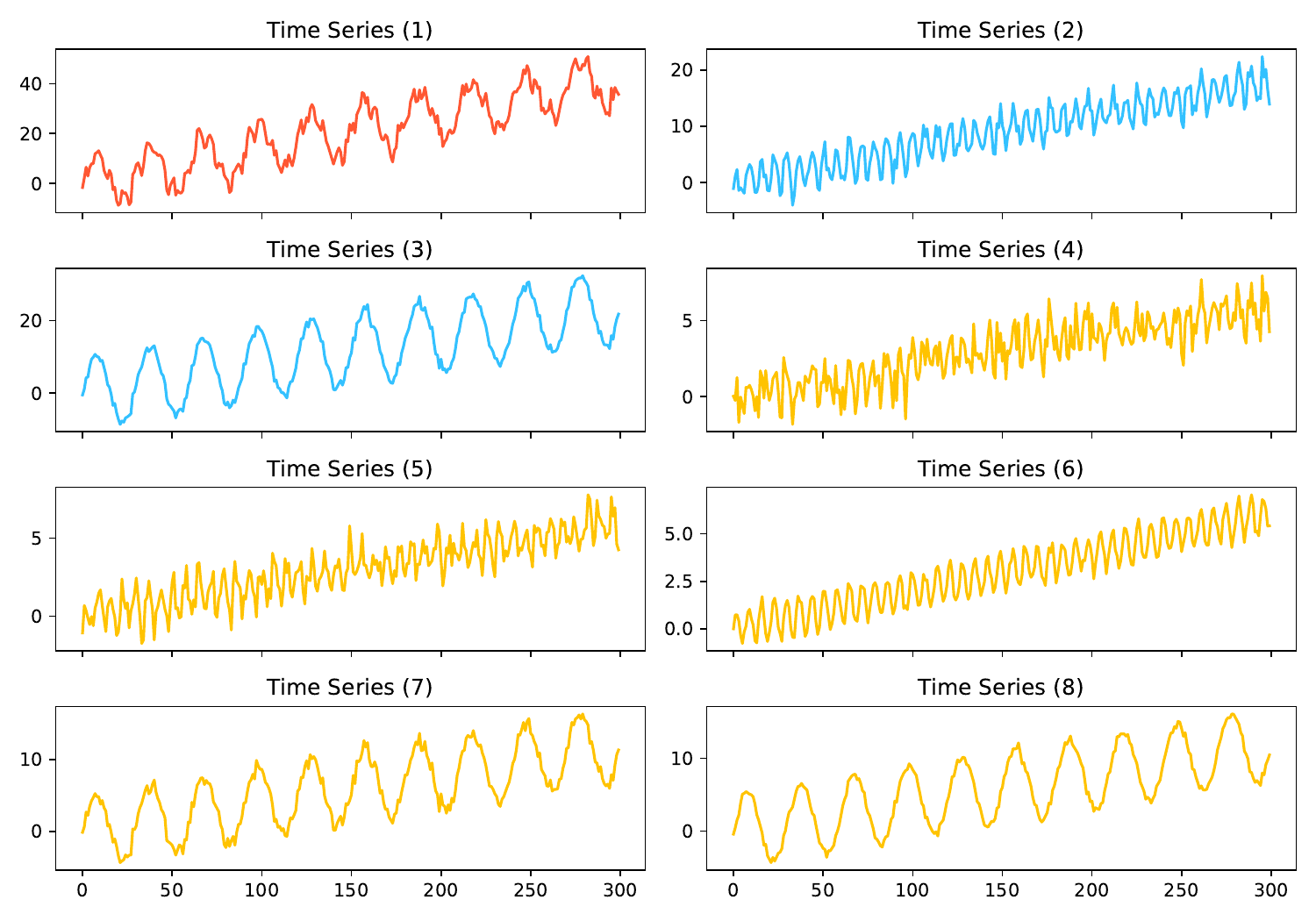}
    \caption{A simulation of the synthetic hierarchical time series with $\alpha = 1$ and a length of 300. Time series belonging to the same hierarchical level are represented using the same color.}
    \label{fig:toy_hts}
\end{figure}

\paragraph{Generation Process.}
The hierarchy contains $d=8$ series distributed over three levels: one top-level series, two middle-level series, and $d_b=5$ bottom-level series, with the structure shown in \Cref{fig:hts_example}. Each bottom-level series $i$ is generated using the additive decomposition model \citep{hyndman_forecasting_2018}
\[
    y_t^{(i)}=a^{(i)}t+b^{(i)}\sin\!\left(\frac{2\pi t}{p^{(i)}}\right)+\epsilon_t^{(i)}.
\]
Series sharing a parent have identical trend and seasonal parameters. Specifically, $(a^{(i)},b^{(i)},p^{(i)})=(0.02,1,7)$ for $i\in\set{4,5,6}$ and $(a^{(i)},b^{(i)},p^{(i)})=(0.04,5,30)$ for $i\in\set{7,8}$. Let $\xi_t=[\epsilon_t^{(4)},\dots,\epsilon_t^{(8)}]^{\trans}\in\reals^5$. We draw
\[
    \xi_t\sim\mathcal N(0,\Sigma_\xi),
    \qquad
    \Sigma_\xi=\mat Q(\alpha\mat\Lambda)\mat Q^{\trans},
    \qquad
    \mat\Lambda=\diag(1,0.5,0.1,0.05,0.01),
\]
where $\mat Q$ is orthogonal and $\alpha\geq0$ controls the noise scale. In particular,
\[
    \expes{\norm{\mat S\xi_t}{\mathcal Y}^2}
    =\alpha\trace{\mat S\mat Q\mat\Lambda\mat Q^{\trans}\mat S^{\trans}}.
\]
Writing $y_t^b=[y_t^{(4)},\dots,y_t^{(8)}]^{\trans}$, the aggregate series are generated as $y_t^{(i)}=\mat S_{i,:}y_t^b$ for $i\in\set{1,2,3}$.

\Cref{fig:toy_hts} shows one realization of the synthetic dataset.

\paragraph{Sensitivity to the covariance scale.}

\begin{figure}[htbp!]
    \centering
    \begin{subfigure}{0.45\linewidth}
        \centering
        \includegraphics[width=\linewidth]{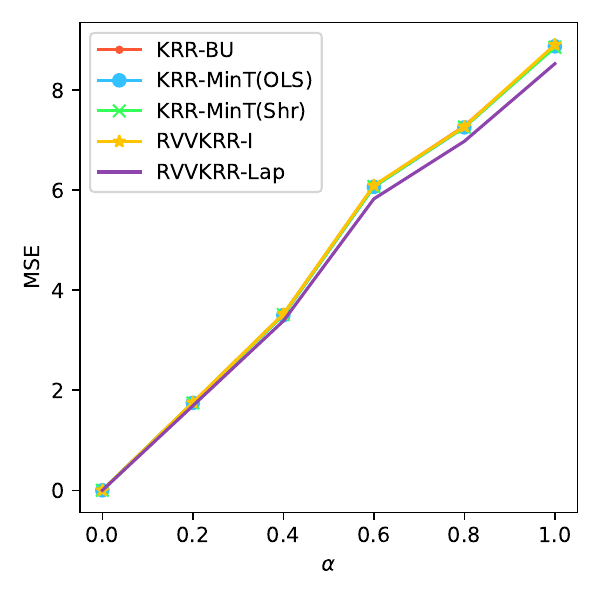}
        \caption{$T=80$.}
    \end{subfigure}
    \hfill
    \begin{subfigure}{0.45\linewidth}
        \centering
        \includegraphics[width=\linewidth]{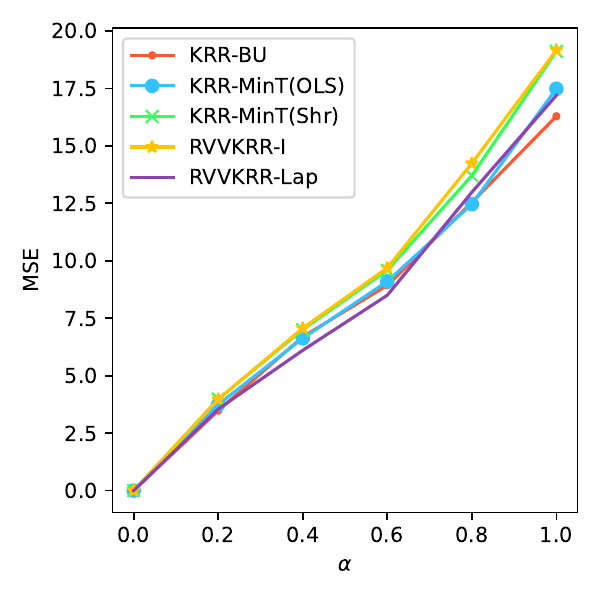}
        \caption{$T=40$.}
    \end{subfigure}
    \caption{All-level MSE on the synthetic dataset as a function of the covariance scaling parameter $\alpha$.}
    \label{fig:toy_alpha_res}
\end{figure}

We additionally study the effect of the covariance scaling parameter $\alpha$, varying it over $[0,1]$ for training lengths $T=80$ and $T=40$. We run 32 Monte Carlo simulations for each configuration, using the same validation and test lengths as in the main experiment.

\Cref{fig:toy_alpha_res} shows a quasi-linear relationship between $\alpha$ and the MSE. RVVKRR-Lap consistently outperforms the other methods for all values of $\alpha$ when $T=80$; when $T=40$, this holds only for $\alpha<0.8$.

\subsection{TourismLarge dataset}
\label{app:tourism}

\paragraph{Experimental Setting.}
Hyperparameters are selected using three-fold expanding-window cross-validation, with training/validation window lengths of $108/24$, $132/24$, and $156/24$. Independent KRR selects a separate ridge parameter $\lambda_r$ for each series from 64 logarithmically spaced values between $10^{-5}$ and $10^8$. For each RVVKRR projection, $\rho$ is selected from $\{10^{-3},3\times10^{-3},10^{-2},3\times10^{-2},10^{-1}\}$ by minimizing the mean Overall MSE across the three validation folds. This selects $\rho=0.03$ for MinT (OLS) and $\rho=0.1$ for both MinT (shrinkage) and MinT (WLS-var). The covariance proxies for the latter two are estimated from in-sample residuals of auxiliary independent KRR predictors within each training window. No feature standardization is applied.

\begin{table}[htbp!]
    \centering
    \caption{Full comparison of test MSE at each aggregation level of TourismLarge, reported in millions ($10^6$). All sums the level-wise MSEs. $\dagger$ denotes unreconciled forecasts. Lower is better; column-wise minima are bold.}
    \begin{tabular}{lcccccc}
        \toprule
        Method & Ctry. & State & Zone & Reg. & Purp. & All \\
        \midrule
        \multicolumn{7}{@{}l}{\textit{OLS base forecasts}} \\
        Unreconciled $\dagger$ & 3.33 & 1.98 & 1.55 & 1.48 & 1.32 & 9.66 \\
        Bottom-Up & 8.97 & 2.79 & 1.88 & 1.52 & 1.32 & 16.48 \\
        Top-Down (AHP) & 3.33 & 4.00 & 2.39 & 2.06 & 1.91 & 13.69 \\
        Top-Down (PHA) & 3.33 & 3.95 & 2.38 & 2.05 & 1.91 & 13.61 \\
        MinT (OLS) & 3.48 & 1.77 & 1.47 & 1.32 & 1.28 & 9.32 \\
        MinT (shrinkage) & 5.33 & 1.97 & 1.49 & 1.30 & 1.24 & 11.33 \\
        MinT (WLS-var) & 6.59 & 2.19 & 1.56 & 1.33 & 1.25 & 12.92 \\
        ERM (closed form) & 4.03 & 2.16 & 2.10 & 1.98 & 2.04 & 12.32 \\
        \midrule
        \multicolumn{7}{@{}l}{\textit{KRR base forecasts}} \\
        Unreconciled $\dagger$ & 2.71 & 1.69 & 1.30 & 1.19 & 1.13 & 8.02 \\
        Bottom-Up & 8.06 & 2.53 & 1.53 & 1.21 & 1.13 & 14.46 \\
        Top-Down (AHP) & 2.71 & 3.94 & 2.38 & 2.05 & 1.92 & 13.00 \\
        Top-Down (PHA) & 2.71 & 3.89 & 2.37 & 2.04 & 1.92 & 12.94 \\
        MinT (OLS) & 2.72 & 1.58 & 1.25 & 1.12 & 1.11 & 7.78 \\
        MinT (shrinkage) & 3.05 & 1.60 & \textbf{1.21} & \textbf{1.08} & \textbf{1.06} & 7.99 \\
        MinT (WLS-var) & 4.56 & 1.83 & 1.29 & 1.11 & 1.09 & 9.89 \\
        ERM (closed form) & 5.36 & 2.50 & 2.61 & 2.57 & 2.52 & 15.57 \\
        \midrule
        \multicolumn{7}{@{}l}{\textit{RVVKRR forecasts}} \\
        MinT (OLS), $\rho=0$ & 2.75 & 1.55 & 1.23 & 1.10 & 1.10 & 7.73 \\
        MinT (OLS), $\rho=0.03$ & 2.74 & 1.54 & 1.24 & 1.11 & 1.10 & 7.72 \\
        MinT (shrinkage), $\rho=0$ & \textbf{2.01} & \textbf{1.51} & 1.28 & 1.23 & 1.58 & \textbf{7.61} \\
        MinT (shrinkage), $\rho=0.1$ & 2.02 & 1.52 & 1.30 & 1.24 & 1.57 & 7.64 \\
        MinT (WLS-var), $\rho=0$ & 2.52 & 1.81 & 1.45 & 1.40 & 1.75 & 8.93 \\
        MinT (WLS-var), $\rho=0.1$ & 2.39 & 1.83 & 1.49 & 1.44 & 1.71 & 8.85 \\
        \bottomrule
    \end{tabular}
    \label{tab:tourism_full_results}
\end{table}

\subsection{Traffic Dataset}
\label{app:traffic}

\paragraph{Experimental Setting.}
The largest lag used as the feature is 28, leaving 338 observations from the original 366-day sample. We reserve the last 28 observations for testing and use the preceding 310 for model selection and fitting. Hyperparameters are selected using three-fold expanding-window cross-validation by minimizing the MSE. Each validation window contains 28 observations, giving training/validation window lengths of $226/28$, $254/28$, and $282/28$. After selection, every model is refitted on all 310 pre-test observations.

Independent KRR selects a separate ridge parameter $\lambda_r$ for each series $r$ from 64 logarithmically spaced values between $10^{-5}$ and $10^8$. The same selected parameters are used by the auxiliary independent KRR fits that estimate the residual-based MinT projections.
Independent OLS uses the same features and has no hyperparameter; the cross-validation results are used only for diagnostic purposes.

For each RVVKRR variant, the global regularization parameter $\lambda$ is selected by the cross-validation over
$\set{10^{-4},10^{-3},10^{-2},10^{-1},1,10,10^2,10^3}$. All three variants select $\lambda=1$. Covariance-based projections are estimated from in-sample residuals using only the current fold’s training window. They are re-estimated using all 310 pre-test observations for the final fit.

\begin{table}[htbp!]
    \centering
    \caption{Full comparison of test MSE at each aggregation level of Traffic dataset, reported in thousands ($10^3$). All sums the level-wise MSEs. $\dagger$ denotes unreconciled forecasts. Lower is better; column-wise minima are bold.}
    \begin{tabular}{lccccc}
        \toprule
        Method & L1 & L2 & L3 & Bottom & All \\
        \midrule
        \multicolumn{6}{@{}l}{\textit{OLS base forecasts}} \\
        Unreconciled $\dagger$ & 8.43 & 4.87 & 3.03 & 0.60 & 16.93 \\
        Bottom-Up & 12.33 & 6.46 & 3.49 & 0.60 & 22.88 \\
        Top-Down (AHP) & 8.43 & 4.56 & 3.04 & 1.33 & 17.36 \\
        Top-Down (PHA) & 8.43 & 4.57 & 3.06 & 1.32 & 17.39 \\
        MinT (OLS) & 8.41 & 4.90 & 2.85 & 0.58 & 16.74 \\
        MinT (shrinkage) & 8.62 & 4.62 & 2.55 & \textbf{0.57} & 16.36 \\
        MinT (WLS-var) & 10.73 & 5.70 & 3.12 & 0.60 & 20.15 \\
        ERM (closed form) & 33.78 & 17.33 & 8.96 & 1.38 & 61.45 \\
        \midrule
        \multicolumn{6}{@{}l}{\textit{KRR base forecasts}} \\
        Unreconciled $\dagger$ & 8.95 & 5.05 & 2.84 & 0.63 & 17.47 \\
        Bottom-Up & 13.30 & 6.91 & 3.74 & 0.63 & 24.58 \\
        Top-Down (AHP) & 8.95 & 4.81 & 3.17 & 1.33 & 18.27 \\
        Top-Down (PHA) & 8.95 & 4.83 & 3.19 & 1.33 & 18.29 \\
        MinT (OLS) & 9.01 & 4.95 & 2.85 & 0.61 & 17.42 \\
        MinT (shrinkage) & 9.16 & 4.83 & 2.69 & 0.60 & 17.28 \\
        MinT (WLS-var) & 11.68 & 6.11 & 3.34 & 0.62 & 21.76 \\
        ERM (closed form) & 72.00 & 36.45 & 18.62 & 1.34 & 128.41 \\
        \midrule
        \multicolumn{6}{@{}l}{\textit{RVVKRR forecasts}} \\
        MinT (OLS) & 8.99 & 4.93 & 2.79 & 0.61 & 17.32 \\
        MinT (shrinkage) & 3.84 & 2.14 & 2.20 & 3.39 & 11.57 \\
        MinT (WLS-var) & \textbf{0.32} & \textbf{0.24} & \textbf{0.44} & 3.02 & \textbf{4.03} \\
        \bottomrule
    \end{tabular}
    \label{tab:traffic_full_results}
\end{table}